\documentclass[anon,12pt]{colt2019} 

\usepackage{color}
\usepackage{setspace}
\usepackage{graphics}
\usepackage{float}
\usepackage{caption}
\usepackage{subcaption}

\newcommand{\Z}{\mathbb{Z}}
\newcommand{\R}{\mathbb{R}}

\newcommand{\B}{\mathcal{B}}

\newcommand{\D}{\mathcal{D}}

\newcommand{\PPB}{\mathcal{P}(\mathcal{B})}
\newcommand{\PPBbar}{\overline{\mathcal{P}}(\mathcal{B})}

\newcommand{\CBONE}{\mathcal{C}^{1}_{\PPB}}
\newcommand{\CBZERO}{\mathcal{C}^{0}_{\PPB}}

\newcommand{\hx}{\hat{x}}
\newcommand{\hz}{\hat{z}}

\newcommand{\bias}{p}

\newcommand{\CN}{\tau_f}

\DeclareMathOperator*{\vol}{Vol}

\title[]{A lattice-based approach\\ to the expressivity of deep ReLU neural networks}
\usepackage{times}

\author{\Name{Vincent Corlay} \Email{v.corlay@fr.merce.mee.com}\\
 \addr{Telecom ParisTech, Paris, \\ Mitsubishi Electric R\&D Centre Europe, Rennes.}
 \AND
 \Name{Joseph J. Boutros} 
  \addr{Texas A\&M University, Doha.}
 \AND
 \Name{Philippe Ciblat} 
 \addr Telecom ParisTech, Paris.
\AND
 \Name{Lo\"ic Brunel} 
 \addr Mitsubishi Electric R\&D Centre Europe, Rennes.
}

\begin{document}

\maketitle


\begin{abstract}%
We present new families of continuous  piecewise  linear (CPWL) 
functions in $\R^n$ having a number of affine pieces growing exponentially in $n$. 
We show that these functions can be seen as the high-dimensional generalization 
of the triangle wave function used by Telgarsky in 2016. 
We prove that they can be computed by ReLU networks 
with quadratic depth and linear width in the space dimension.
We also investigate the approximation error of one of these functions by shallower networks
and prove a separation result.
The main difference between our functions and other constructions is their practical interest: they arise in the scope of channel coding. 
Hence, computing such functions amounts to performing a decoding operation. 
\end{abstract}


\section{Introduction and Main Results}
This paper follows two recent articles (but is self-contained), \cite{Corlay2018} and \cite{Corlay2019}, 
where we jointly study point lattices in Euclidean space and neural networks.
Our aim is twofold. Firstly, apply neural networks paradigm to find new
efficient decoding algorithms.
Secondly, contribute to the understanding of the efficiency of deep learning.
In this work, we emphasize the second aspect and highlight a direct contribution
of lattice coding theory to deep learning.

More  specifically, we  focus on  the  {\em expressive  power of  deep
  neural networks}.  Typically,  the goal of this line  of research is
to show that there exist functions  that can be well approximated by a
deep  network  with  a  polynomial number  of  parameters  whereas  an
exponential number  of parameters is  required for a  shallow network.
Many   results    in   the   literature    like   \cite{Montufar2014},
\cite{Telgarsky2016}, \cite{Arora2018}  utilize functions that  can be
addressed  via  ``conventional  methods"  (i.e. not  via  deep  neural
networks):  they  are  mostly  based on  one  dimensional  approaches.
Please,  see Appendices~\ref{App_expre_power_NN}~and~\ref{App_Telgarsky}
for a  survey on recent  results on this  topic and an  elucidation of
main  techniques.   Functions associated  to  point  lattices are  too
complicated to be computed via conventional methods, thus illustrating
the benefit  of both  neural networks  and depth.   They arise  in the
context of the sphere  packing problem and lattices \cite{Conway1999}.
We  argue  that  these  functions enlighten  the  missing  dimensional
dependency  in  the  bound  of  \cite{Telgarsky2016}.   Moreover,  for
dimensional dependent separation bounds  to hold in higher dimensions,
our  investigation highlights  the need  for sophisticated  functions.
Such functions can be found thanks to {\em dense} lattices.

Short-length error-correcting codes used to protect digital information transmission
are discrete sets mainly built via Algebra: 
e.g. vector spaces over finite fields or modules over rings.
The decoding operation in a discrete set consists in finding the closest element
to a noisy received signal. This is a classification problem.

The channel  coding community  recently started  to use  deep learning
techniques  to tackle  this classification  problem.  The  interest in
deep   learning  for   channel   coding   is  growing   exponentially.
However, the first  attempts to perform decoding
operations with ``raw" neural  networks (i.e. without using underlying
graph  structures  of  existing  sub-optimal algorithms,  as  done  in
\cite{Nachmani2018}) were  unsuccessful. For instance,  an exponential
number of  neurons in  the network is  needed in  \cite{Gruber2017} to
achieve satisfactory  performance.  So far,  it was not  clear whether
such a behavior is  due  to  an  unadapted  learning  algorithm  or  a
consequence  of a poor  function class.   This work  is a  first
theoretical step towards a better  understanding of the function class
that should be more suitable for usage in these decoding problems.

In \cite{Corlay2019}, we rigorously  presented the duality between the
decoding operation for lattices,  the so-called closest vector problem
(CVP), and  a classification problem in  the fundamental parallelotope
with  a CPWL  function  defining the  decoding boundary.   Preliminary
results for one of the most famous root lattices, namely $A_n$, were also
presented:  for a  given basis  of  $A_n$, the  function defining  the
boundary has $\Omega(2^n)$ affine pieces.   We managed to reduce the
number   of  pieces   to  be   computed  down to  a   linear  number via
$\mathcal{O}(n^2)$   reflections   with   respect  to   the   bisector
hyperplane  of pairs of  vectors in the lattice basis.
Hence,  the  evaluation of this decision boundary function can  be
performed  by a  ReLU network  of depth  $\mathcal{O}(n^2)$ and  width
$\mathcal{O}(n)$.  We  also proved that  a ReLU network with  only one
hidden layer requires $\Omega\left(2^{n}\right)$  neurons to compute this
function.  We did not quantify the approximation error.

\vspace{-1mm}
\subsection{Main Results}
In this paper, we complete the initial results of \cite{Corlay2019} with the following contributions:
\begin{enumerate}
\item We show that the CPWL boundary function $f$, obtained from $A_n$,
is a $n$-dimensional generalization of the triangle wave function used by \cite{Telgarsky2016}.

\item We investigate the approximation error of $f$ by a function $g$ having a restricted number of pieces. 
We prove that, for a large enough dimension and within the fundamental parallelotope, 
$f$ can be approximated by a one-neuron linear network with a negligible error.
This emphasize the need for more sophisticated functions to illustrate the benefit of depth in high dimensions for a fixed size 
of the domain of $f$.

\item However, if $f$ is not limited to this parallelotope but to a larger compact set, whose size increases exponentially
with the depth of the network used for approximation, we get a separation result. 
Theorem~\ref{th_zero_error} (with the parameter $M=n^2$) has the following consequence: there exists a function $f:\R^{n-1} \rightarrow \R$
computed by a standard ReLU neural network in $\mathcal{O}(n^2)$ layers and $\mathcal{O}(n^3)$ neurons where
any function $g$ computed by a ReLU neural network with $\leq n$ layers and $\leq 2^{n-1}$ neurons induces a $L_1$ 
approximation error
$||f-g||_1 = \Omega(2^{(n-1)^3-n\log_2(n)})$.

\item We present new sophisticated CPWL functions arising from root lattices
$D_n$, $n \ge 2$, and $E_n$, $6 \leq n \leq 8$. 
The exact numbers of pieces of these functions are provided by explicit formulas.
These numbers are exponential in the space dimension. 

\item We show that each of these functions can be computed by a ReLU network
with polynomial depth and linear width. This is achieved by $folding$ the input space: 
i.e. we perform reflections in the input space as pre-processing. 
After a polynomial number of reflections, the functions can be evaluated 
by computing a number of affine functions growing only linearly in the space dimension.
\end{enumerate}



\vspace{-5mm}
\section{Lattices, Polytopes, and the decision boundary function}
\label{sec_nota}
This section is highly inspired from \cite{Corlay2019}. 
We establish the notations and state existing results used in the sequel. 

\vspace{-3mm}
\subsection{Lattices and polytopes}

A lattice $\Lambda$ is a discrete additive subgroup of $\R^n$.
For a rank-$n$ lattice in $R^n$, the rows of a $n\times n$ generator matrix $G$ constitute
a basis of $\Lambda$ and any lattice point $x$ is obtained via $x=zG$, where $z \in \Z^n$.
If needed, $x$ also denotes the corresponding vector.
Also, $\Gamma = GG^{T}$ is the Gram matrix (see Appendix~\ref{App_latt} for more details on $\Gamma$).
For a given basis $\mathcal{B}=\{ b_i \}_{i=1}^n$, $\PPB$ denotes
the fundamental parallelotope of $\Lambda$ and $\mathcal{V}(x)$ the Voronoi cell
of a lattice point $x$ (see Appendix~\ref{App_latt} for formal definitions of these fundamental regions of a lattice).
The minimum Euclidean distance of $\Lambda$ is $d_{min}(\Lambda)=2\rho$, where $\rho$ is the packing radius.

A vector  $v \in \Lambda$ is  called Voronoi vector if  the half-space
$\{y \in \mathbb{R}^{n} \ : \ y  \cdot v \leq \frac{1}{2}v \cdot v \}$
has a non-empty intersection with $\mathcal{V}(0)$. The vector is said
relevant  if   the  intersection  is  an   $n-1$-dimensional  face  of
$\mathcal{V}(0)$. We denote  by  $\tau_{f}$ the  number of  relevant
Voronoi  vectors, referred  to  in  the sequel as  the {\em  Voronoi  number}
of the lattice. The Voronoi number and the kissing number are equal for root lattices.
The set of relevant Voronoi vectors is denoted $\CN(0)$.
The set of lattice points having a common Voronoi facet with $x \in \Lambda$
becomes $\CN(x)=\CN(0)+x$.


Lattice decoding refers to the method of finding the closest lattice point,
the closest in Euclidean distance sense.
This problem is also known as the closest vector problem.
Our functions are mostly studied in the compact region $\PPB$,
thus it is important to characterize $\PPB$ as made below.

Let $\PPBbar$ be the topological closure of $\PPB$.
A $k$-dimensional element of $\PPBbar\setminus \PPB$
is referred to as $k$-face of $\PPB$.
There are $2^n$ 0-faces, called corners or vertices.
This set of corners is denoted $\mathcal{C}_{\PPB}$.
Moreover, the subset of $\mathcal{C}_{\PPB}$ obtained with $z_i=1$ is
$\mathcal{C}^{i,1}_{\PPB}$ and $\mathcal{C}^{i,0}_{\PPB}$ for $z_i=0$. 
The remaining faces of $\PPB$ are parallelotopes. For instance, a $n-1$-dimensional facet of $\PPB$ 
is itself a parallelotope of dimension $n-1$ defined by $n-1$ vectors of $\B$. 
Throughout the paper, the term facet refers to a $n-1$-face.
Also, for the sake of simplicity, $\PPB$ refers to $\PPBbar$.

The following definition ensures optimality when decoding via the decision boundary in $\PPB$.
\begin{definition}
\label{def_Voronoi-reduced}
Let $\B$ be the $\Z$-basis of a rank-$n$ lattice $\Lambda$ in~$\R^n$.
$\B$ is said Voronoi-reduced (VR) if, for any point $y \in \PPB$,
the closest lattice point $\hx$ to $y$ is one of the $2^n$ corners of $\PPB$,
i.e. $\hx=\hz G$ where $\hz \in \{0, 1\}^n$.
\end{definition}

Some of the above conditions are relaxed to yield the less restrictive 
definition of a semi-Voronoi-reduced (SVR) basis. 
A rigorous understanding of this definition is not necessary to grasp the main ideas of the  paper.
While a SVR basis does not enable perfect decoding, 
it ensures the existence of a decision boundary function (described below). 
The formal definition of a SVR basis is provided in Appendix~\ref{App_latt}.

A convex polytope (or convex polyhedron) is defined as the intersection of a finite number of half-spaces bounded by hyperplanes (\cite{Coxeter1973}):
\[
P_{o}=\{x \in \R^{n} : \ xA \leq b, \ A \in \R^{n \times m}, \ b \in \mathbb{R}^{m}\}.
\]
In this paper, parallelotopes are not the only polytopes considered as we also use simplices.
A $i$-simplex associated with $\mathcal{U}=\{u_{j}\}_{j=0}^{i}$ is given by
\begin{align}
\label{eq_fake_simp}
\mathcal{S}(\mathcal{U})= \{ & y \in \R^{n} :  \ y= \sum_{j=1}^{i} \alpha_{j}(u_j-u_0),
\sum_{j=1}^{i} \alpha_{j} \leq1, \ \alpha_{j} \geq 0 \  \forall \ j   \}.
\end{align}
By abuse of terminology, the definition of \eqref{eq_fake_simp} is maintained even if the 
vectors in the set $\mathcal{U}$ are not affinely independent. In this latter case, we refer to $i$ as the size of the simplex whereas it is its dimension otherwise.
It is clear that the corners of $S(\mathcal{U})$ are the $i+1$ points of $\mathcal{U}$. 


We say that a function $g : \mathbb{R}^{n-1} \rightarrow \mathbb{R}$ is continuous piecewise linear (CPWL) 
if there exists a finite set of polytopes covering $\mathbb{R}^{n-1}$,
and $g$ is affine over each polytope. 
The number of pieces of $g$ is the number of distinct polytopes partitioning its domain.

$\vee$ and $\wedge$ denote respectively the maximum and the minimum operator. 
We define a convex (resp. concave) CPWL function formed by a set of affine functions
related by the operator $\vee$ (resp. $\wedge$). 
If $\{g_{k}\}$ is a set of $K$ affine functions,
the function $f=g_{1} \vee ... \vee g_{K}$ is CPWL and convex.

\vspace{-3mm}
\subsection{The decision boundary function}
The notion of decision boundary function for a lattice was introduced in \cite{Corlay2019}.
Given a VR basis, after translating the point to be decoded inside $\PPB$ to get a point $y \in \PPB$,
the decoder proceeds in estimating each $z_i$-component separately.
The idea is to compute the position of $y$ relative to a boundary
to guess whether $z_i=0$, i.e. the closest lattice point belongs to $\mathcal{C}^{i,0}_{\PPB}$,
or $z_i=1$ when the closest lattice point is in $\mathcal{C}^{i,1}_{\PPB}$.
This boundary cuts $\PPB$ into two regions.
It is composed of Voronoi facets of the corner points. 
\textbf{For the rest of the paper, without loss of generality, the integer coordinate to be decoded is $z_1$}.
Also, to lighten the notations $\mathcal{C}^{1,0}_{\PPB}=\mathcal{C}^{0}_{\PPB}$ and $\mathcal{C}^{1,1}_{\PPB}=\mathcal{C}^{1}_{\PPB}$.

Any Voronoi facet is contained in a boundary hyperplane orthogonal to a vector $v_j$, the equation of which is: 
\begin{equation}
\{ y \in \R^n : \ y \cdot v_j - \bias_j =0  \}.
\end{equation}
Any boundary hyperplane contains the Voronoi facet of a point $x \in \CBONE$ and a point from $\CN(x) \cap \CBZERO$ 
(i.e. the Voronoi facet between $x$ and any point in $\CN(x) \cap \CBZERO$  lies in a boundary hyperplane).
The decision boundary cutting $\PPB$ into two regions, with $\mathcal{C}^0_{\PPB}$ on one side
and $\mathcal{C}^1_{\PPB}$ on the other side, is the union of these Voronoi facets.
Each facet can be defined by an affine function over a compact subset of $\mathbb{R}^{n-1}$
and the decision boundary is locally described by one of these functions. 

Let $\{ e_i \}_{i=1}^n$ be the canonical orthonormal basis of the vector space $\R^n$.
For $y \in \R^n$, the $i$-th coordinate is $y_i=y \cdot e_i$.
Denote $\tilde{y}=(y_2, \ldots, y_n) \in \R^{n-1}$ and let $\mathcal{H} = \{ h_j \}$ be the set of affine functions involved in the decision boundary.
The affine boundary function $h_{j}:\R^{n-1} \rightarrow \mathbb{R}$ is
\begin{equation}
\label{equ_hj}
h_{j}(\tilde{y})=y_{1}= \bigg( \bias_{j} -\sum_{k \neq 1} y_{k}v_{j}^{k} \bigg)/v_{j}^{1},
\end{equation}
where $p_j$ is a bias and $v_{j}^{k}$ is the $k$-th component of vector $v_{j}$.
For the sake of simplicity, in the sequel $h_{j}$ shall denote the function defined
in (\ref{equ_hj}) or its associated hyperplane $\{ y \in \R^n : \ y \cdot v_j - \bias_j =0  \}$
depending on the context. 
The following theorem shows the existence of such a boundary function for a VR or SVR basis.
\begin{theorem}
\label{th_func_VR}(Proved in \cite{Corlay2019})
Consider a lattice defined by a VR or a SVR basis $\B=\{b_i\}_{i=1}^n$.
Suppose that the $n-1$ points $\mathcal{B} \backslash \{ b_1\}$
belong to the hyperplane $\{y \in \R^n : \ y \cdot e_1 =0\}$.
Then, the decision boundary is given by a CPWL function
$f:\R^{n-1} \rightarrow~\R$, expressed as
\begin{equation}
\label{eq_boundary_funct}
f(\tilde{y}) = \wedge_{m=1}^{M}\{\vee_{k=1}^{l_m}g_{m,k}(\tilde{y}) \},
\end{equation}
where $g_{m,k} \in \mathcal{H}$, $1 \leq l_m< \tau_f  $, and  $1 \leq M \leq 2^{n-1}$.
\end{theorem}

From now on, {\bf the default orientation of the basis with respect to the
canonical axes of $\R^n$ is assumed to be the one of Theorem~\ref{th_func_VR}}.
We call $f$ the decision boundary function.
The domain of $f$ (its input space) is $\D(\B) \subset \mathbb{R}^{n-1}$.
The domain $\D(\B)$ is the topological closure of the projection of $\PPB$ on the hyperplane $\{e_i\}_{i=2}^n$.
It is a bounded polyhedron that can be partitioned into convex (and thus connected) regions which we call linear regions. 
For any $\tilde{y}$ in one of these regions,
$f$ is described by a unique local affine function $h_{j}$. 
The number of those regions is equal to the number of affine pieces of $f$. 

In the sequel, $L$ denotes the number of layers in the neural network evaluating
the boundary function $f$ defined on $\D(\B)$ and $w$ is the network width. 
Also, $\mathcal{N}(f)$ gives the number of pieces of a CPWL function $f$. 

\section{(In)approximability results for $A_n$}
\label{sec_approx}

Consider a basis for the lattice $A_n$ with all vectors from the first lattice shell.
Also, the angle between any two basis vectors is $\pi/3$.
Let $J_n$ denote the $n \times n$ all-one matrix and $I_n$ the identity matrix.
The Gram matrix is
\begin{equation}
\label{eq_basis_An}
\Gamma_{A_n}=GG^T= J_n+I_n.
\end{equation}

Assume that one is only given $k$ pieces to build a function $g$ approximating $f$,
with $k<\mathcal{N}(f)$. What is the minimum possible approximation error?

The  decision  boundary  function  $f$ for  $A_2$  is  illustrated  on
Figure~\ref{fig_func_A2} by  the thick  yellow line. This  function is
the  same triangle  wave function as the one used to  prove the  main separation
theorem between deep and shallow networks in \cite{Telgarsky2016}.  We
quickly  recall  the  main  ideas   of  his  proof  (a  more  detailed
explanation  is also  available  in Appendix~\ref{App_Telgarsky}).   A
triangle wave  function $f$  with $p$ periods  is considered.   It has
$2p+1$ affine  pieces.  Telgarsky  established a  lower bound  of the
average pointwise disagreement $|f(\tilde{y})  - g(\tilde{y})|$ over a
compact set between this function and a function $g$ having $k$ pieces
where $k<2p+1$.  This is achieved  by summing the triangle areas above
(resp.  below)  the dashed  black line  (see Figure~\ref{fig_func_A2} or Figure~\ref{fig_func_Tel})
whenever $g$ is below (resp. above) this same line.  Indeed, since $g$
has a limited number of pieces, it  can only cross this line a limited
number of times.

\begin{figure}[t]
\centering
\vspace{-5mm}
\includegraphics[scale=0.2]{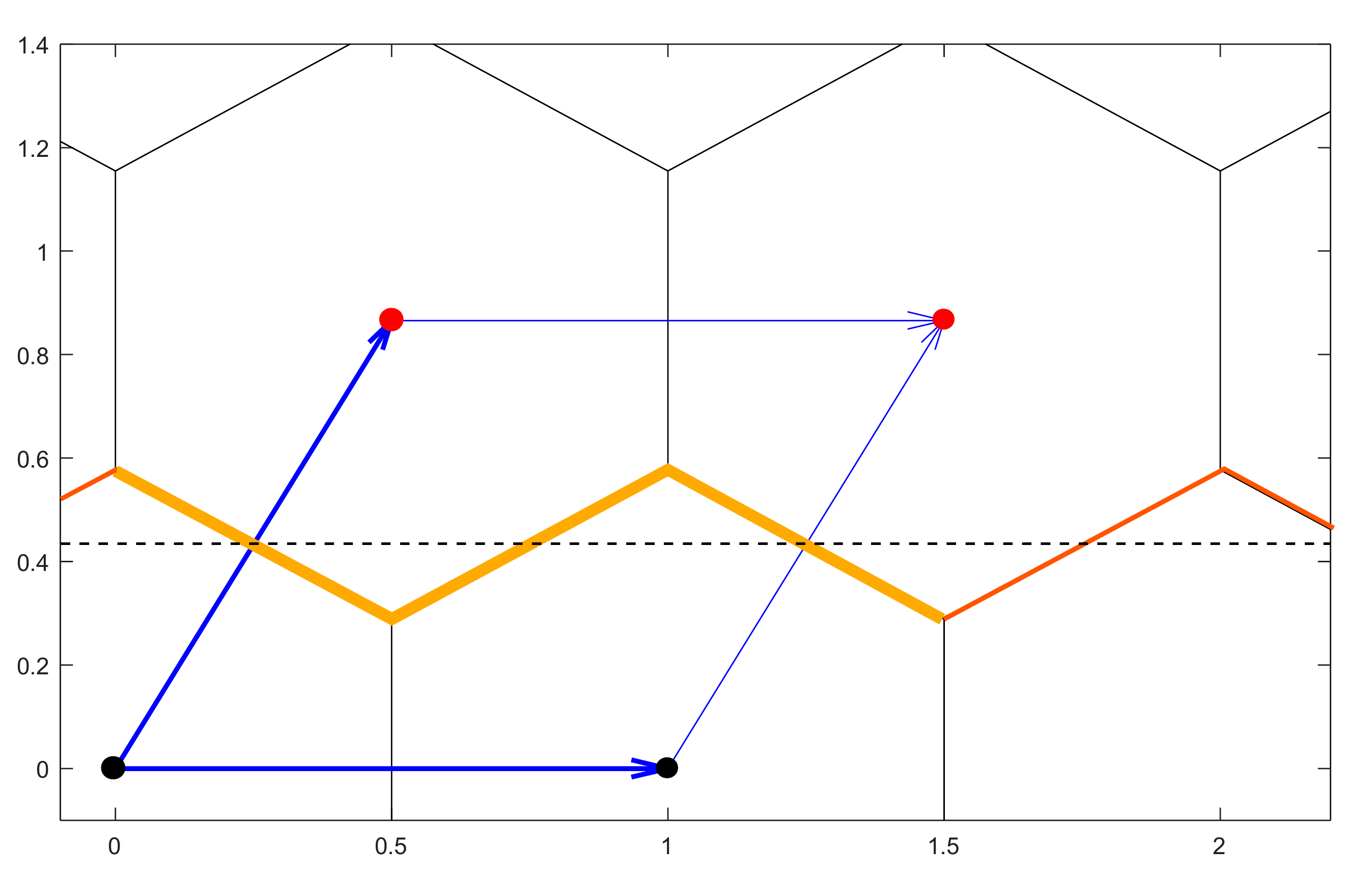}
\caption{$\PPB$ for $A_2$. The thick blue arrows represent the basis.
The CPWL boundary function $f$ defined on $\D(\B)$ is illustrated by the thick yellow line.
The corner points in $\CBONE$ are in red and the corner points in $\CBZERO$ are in black.
$f$ can also be extended to a larger compact set (the thin red lines).}
\label{fig_func_A2}
\end{figure}

Now, what happens  if we consider a similar function  in $\R^3$, where
we replace triangles by tetrahedra?  Such a function, limited to
$\D(B)$,  is  illustrated  on   Figure~\ref{fig_func_A3}.   It  is  the
decision boundary obtained for $A_3$ defined by $\eqref{eq_basis_An}$.
The dashed line of Figure~\ref{fig_func_A2}  should now be replaced by
the plane $\Phi^{n=3}=\{  y \in \R^3 :  \ y \cdot e_1 =  \frac{1}{2}\times(b_1\cdot e_1) \}$.
Similarly to  the triangle wave function, $f^{n=3}$ is
oscillating  around  $\Phi^{n=3}$:  all   pieces  of  $f^{n=3}$  cross
$\Phi^{n=3}$.   Note  that  the  number  of  pieces  is  significantly
increased compared to a simple extension of the triangle wave function
in $\R^3$ (see e.g. Figure~\ref{fig_func_Tel_2}).  Another figure with
$\Phi^{n=3}$      cutting      $f^{n=3}$     is      available      in
Appendix~\ref{App_addi_mat}.   The same  pattern is  observed for  any
space dimension $n$, where the triangles or tetrahedra become $n$-simplices.

\begin{figure}[t]
\centering
\begin{minipage}{.45\textwidth}
\centering
\includegraphics[width=0.9\linewidth]{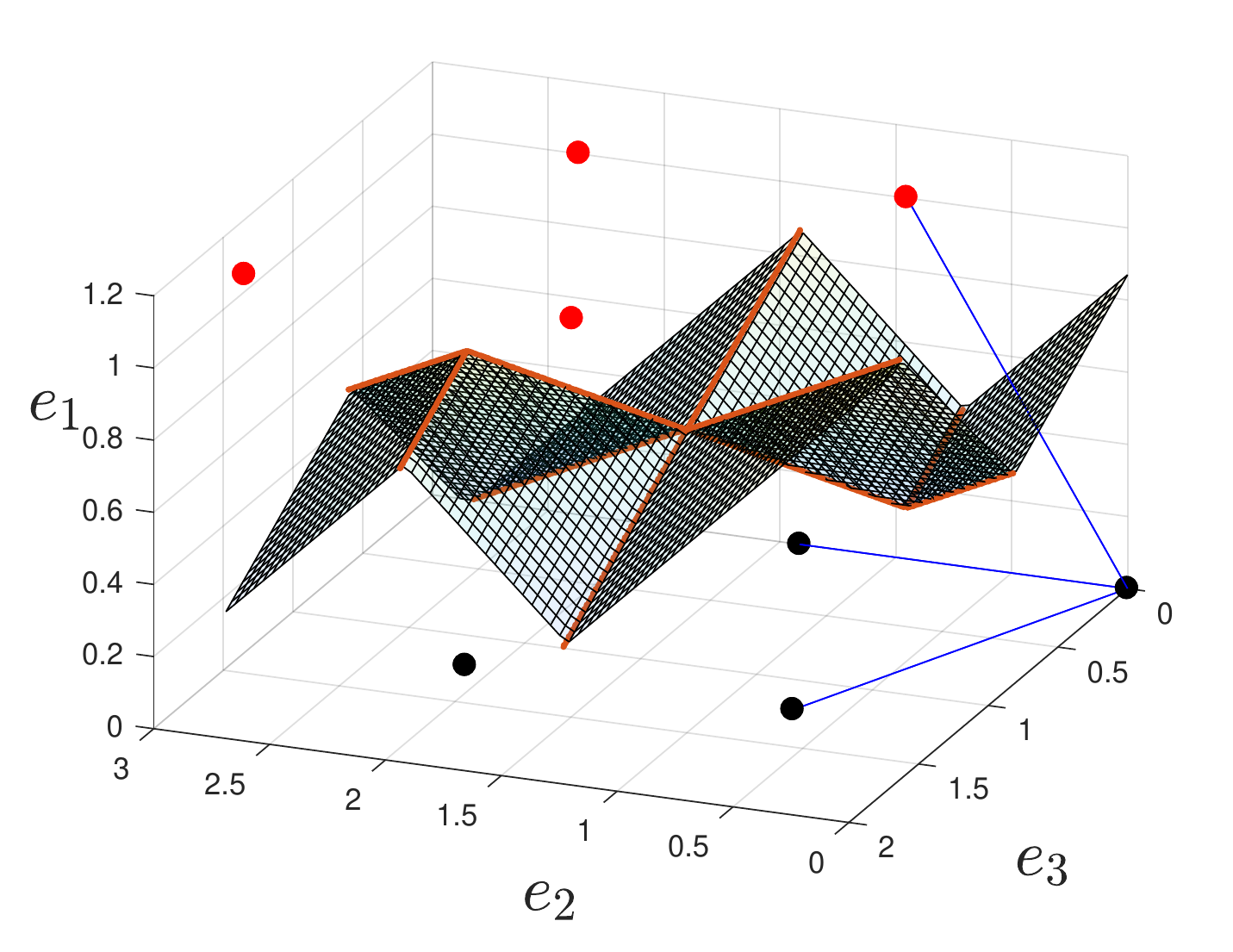}
\caption{CPWL decision boundary function for $A_3$. 
The basis vectors are represented by blue lines. 
The corner points in $\CBONE$ are in red and the corner points in $\CBZERO$ in black.}
\label{fig_func_A3}
\end{minipage}%
\hspace{8mm}
\begin{minipage}{.45\textwidth}
  \centering
   \vspace{-5mm}
   \includegraphics[width=0.8\linewidth,angle=-30]{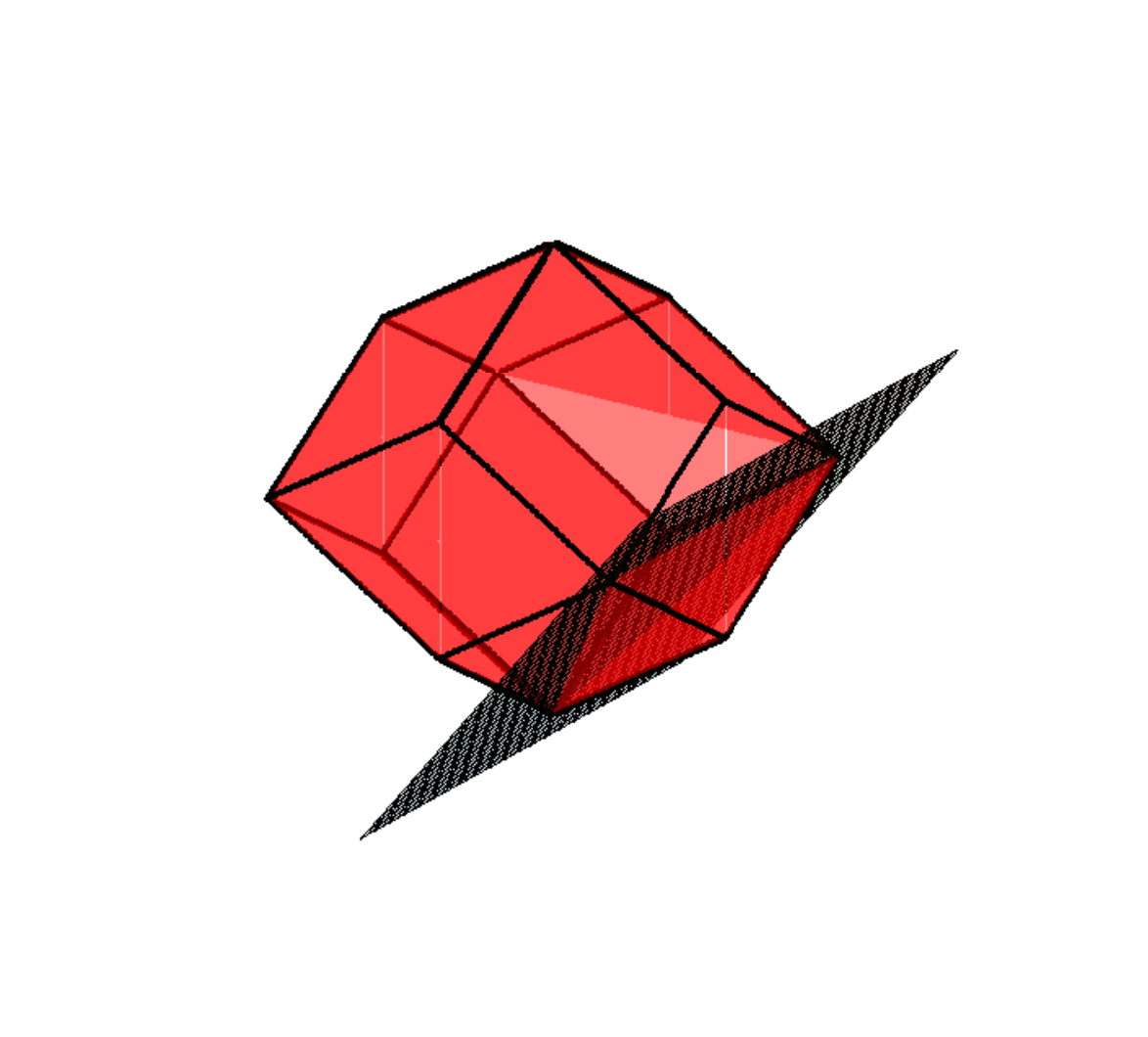}
   \vspace{-20mm}
  \caption{The Voronoi cell of $A_3$ is a rhombic dodecahedron. 
                A subset of the facets of this polytope generates some affine pieces of $f$.
	         The non-truncated tetrahedron is the part of the dodecahedron below the plane $\Phi^{n=3}$.}
  \label{fig_dode}
\end{minipage}
\end{figure}

Consider any convex part of $f$, say $f_m =\vee_{k=1}^{l_m}g_{m,k}$ (see~\eqref{eq_boundary_funct}). 
There are $2^{n-1}$ of such $f_m$. The polytope
\begin{align*}
\{ y \in \PPB : y_1 \ge f_m(\tilde{y}), \ y \cdot e_1 \le \frac{1}{2}\times(b_1 \cdot e_1) \}
\end{align*}
is a truncated simplex due to the limitation of $f$ to $\D(\B)$. 
For all $m$, $1\le m \le 2^{n-1}$, these polytopes are the truncated version of a $n$-simplex.
This simplex is illustrated for $n=3$ on Figure~\ref{fig_dode}.

This same function can be extended to $\R^n$ by periodicity, i.e. same
boundary in $\PPB+x$ as in $\PPB$,  for any lattice point $x$. Indeed,
$\PPB$ is  a fundamental region of  the lattice and one  can perform a
tessellation of  $\R^n$ with  $\PPB$.  This translates  into extending
the  boundary   function  of  (\ref{eq_boundary_funct})   as  follows:
$f(\tilde{y_0})=f(\tilde{y})$ where $y=y_0-x \in \PPB$.  We consider a
set  $\mathcal{P}(\{  b_1,  \alpha  b_2, \alpha  b_3,  \ldots,  \alpha
b_n\})$,  where $\alpha=2^M$  and $M\ge  1$  is an  integer.  The  new
scaled  region has  $2^{M(n-1)}$  copies  of $\mathcal{P}(\B)$.   This
extended function  is defined over the domain $\D(\{  b_1,  \alpha  b_2, \alpha  b_3,  \ldots,  \alpha
b_n\})$, which is the projection of the scaled region on the hyperplane $\{e_i\}_{i=2}^n$. 
If we  let $M$ grow with  $n$, the exponential increase  of the volume
yields a  total number  of pieces superexponential  in $n$.   The next
proposition, showing  that this extended function can  be efficiently  computed by a  deep and
narrow      network,       is      constructively       proved      in
Appendix~\ref{lem_super_exp}.

\vspace{-2mm}
\begin{proposition}
\label{lemma_superexp}
Consider a VR or SVR basis $\B$ defining any lattice and 
its extended decision boundary function defined on the compact set
$\D(\{ b_1, \alpha b_2, \alpha b_3, \ldots, \alpha b_n\})$,
where $\alpha=2^M$. 
Then, the boundary function has $\Omega(2^{M(n-1)})$ pieces
and it can be computed by a ReLU network of width $\max(3(n-1), w)$
and depth $3M+L$ (where $L$ and  $w$ are the parameters of the neural network evaluating
$f$ on $\D(\B)$). 
\end{proposition}

\vspace{-1mm}
The boundary function $f$ limited to $\D(\B)$ and its extension
to $\D(\{ b_1, \alpha b_2, \alpha b_3, \ldots, \alpha b_n\})$
are used to prove approximability and inapproximability results for shallow networks.

\begin{theorem}
\label{th_zero_error}
Consider an $A_n$-lattice basis defined by the Gram matrix~\eqref{eq_basis_An}.
Let $f$ be the decision boundary function.

\begin{enumerate} 
\item Suppose that \eqref{eq_basis_An} is scaled by $(n+1)^{-1/n}$ to get $\text{Vol}(\PPB)=1$.
If $f$ is defined on the compact set $\D(\B)$, there exists an affine function $h$ represented as a linear network with one neuron such that:
\begin{equation}
\underset{n \rightarrow \infty }{\lim} ||f -h||_1 = \underset{n \rightarrow \infty }{\lim} \int_{\D(\B)} |f(\tilde{y})-h(\tilde{y})|d\tilde{y} = 0.
\end{equation}
\item Let $f$ be defined on the compact set $\D(\{ b_1, \alpha b_2, \alpha b_3, \ldots, \alpha b_n\})$,
where $\alpha=2^M$ and $||b_i||~=~\sqrt{2}$, $1 \le i \le n$. For $M$ large enough, any function $g$ that can be computed by a $L$-deep, 
$w$-wide ReLU neural network where  $L \log_2(w)\leq M-n$ has an error 
\begin{equation}
||f -g||_1 = \Omega \left(2^{(n-1)M-n\log_2(n)}\right),
\end{equation}
whereas if $L=3M+\mathcal{O}(n^2)$ and $w=3(n-1)$, $f$ can be computed by the network.
\end{enumerate}
\end{theorem}

\vspace{-1mm}
\textbf{Sketch of proof}  1. Assume that there are $K$ distinct truncated simplices of the form 
$\{ y \in \PPB : y_1 \ge f_m(\tilde{y}), \ y \cdot e_1 \le \frac{1}{2}\times(b_1 \cdot e_1) \}$  or $\{ y \in \PPB : y_1 \le f_m(\tilde{y}), \ y \cdot e_1 \ge \frac{1}{2}\times(b_1 \cdot e_1) \}$. 
The $L_1$ difference between $f$ and the function $h_{\Phi}$, 
defined by the hyperplane $\Phi = \{ y \in \R^n : \ y \cdot e_1  = \frac{1}{2}\times(b_1 \cdot e_1) \}$, 
is bounded from above by the volume of $K$ non-truncated simplices.
Under $\text{Vol}(\PPB)=1$, the volume of a non-truncated simplex is bounded from above by $1/n!$.
There are $K=2^n$ distinct truncated simplices in $\PPB$. 
Hence, the upper-bound is asymptotic to $\frac{1}{\sqrt{2 \pi n} 2^{n \log_2(n/e)-n}}$.

2. We begin with the first part of the second result. The volume of a non-truncated simplex is $\Omega(1/n^n)$.
If $\mathcal{P}(\{ b_1, \alpha b_2, \alpha b_3, \ldots, \alpha b_n\})$ is large enough, 
there are at least as many non-truncated simplex as instance of $\PPB$ in the large compact set: i.e. $2^{M (n-1)}$.
The result is then achieved  by using the fact that no $L$-deep $w$-wide ReLU network with input in $\R^{n-1}$ 
can compute more than $\mathcal{O}(2^{(n-1)L\log_2(w)})$ pieces, combined with the ``crossing" 
argument of \cite{Telgarsky2016}. \\
The second part of the result is a direct consequence of Proposition~\ref{lemma_superexp}, 
where the part of $f$ on $\mathcal{\D}(\B)$ is evaluated via folding with $\mathcal{O}(n^2)$
reflections. 
The formal proof is available in Appendix~\ref{App_proof_theo_zero_error}. $\blacksquare$

\
\begin{corollary}
\label{coro_An}
Consider an $A_n$-lattice basis defined by the Gram matrix~\eqref{eq_basis_An}.
Let $y \in \PPB$ be drawn from a uniform distribution over $\PPB$.
The average error, when decoding the first coordinate $z_1$ of $y$ via the sign of the projection of $y$ on the normal vector to $\Phi$, is $\epsilon <\frac{1}{\sqrt{2 \pi n} 2^{n \log_2(n/e)-n}}$.
\end{corollary}

Note that, despite the first result of Theorem~\ref{th_zero_error} and Corollary~\ref{coro_An}, 
for medium dimensions, there may be an interest to add pieces to the function $g$ approximating $f$ on $\D(\B)$.
Indeed, the decrease in the approximation error might be too slow for some applications 
(e.g. in communications error rates of at least $10^{-5}$ are expected) and could be speed up via additional pieces. 
The second result of the theorem is interesting only for networks of small or medium depth as the size of 
the compact set should increase exponentially for the inapproximability to hold.

On the other hand, if the size of the compact set is fixed, 
any shallow network can approximate the boundary function $f$ for $A_n$ due to the decrease in 
$1/n!$ of the volume of a $n$-simplex. 
Hence, we need more sophisticated functions 
to illustrate the benefit of neural networks and depth in this situation. 
We present such functions in the next section. 

\vspace{-2mm}
\section{Folding-based neural decoding of $D_n$ and $E_n$}
\label{sec_big_section}

We introduce three new CPWL functions. 
For each function, we proceed as follows.
\begin{enumerate}
\vspace{-1mm}
\item We count the number of pieces of the function defined on $\D(\B)$. It is shown to be exponential in the space dimension.
\vspace{-1mm}
\item We prove that  the function can be efficiently computed via $folding$: 
i.e. we perform a quadratic number of reflections on $\tilde{y} \in \D(\B)$ as pre-processing. 
After folding, the function can be evaluated by computing a linear number of affine pieces. 
\vspace{-1mm}
\item We then rely on the strategy detailed in Appendix~\ref{sec_folding_relu} (presented in \cite{Corlay2019}) to show 
how this translates into a ReLU neural network of depth increasing linearly with the number of reflections and a width that is linear in the dimension.
\end{enumerate}
The study of approximation of these functions by shallower networks is not provided in this section and left for future work. 
However, we conjecture that their more complex structure makes them harder to be approximated than the function of the previous section.
Hence, they could potentially be used to show gap theorems without using the oscillatory/periodic construction.
\vspace{-3mm}
\subsection{$D_n$ with the basis of Construction A}
\label{sec_const_A}

$D_n$ can be generated from the parity check code via Construction A (\cite{Conway1999}). 
This leads to a basis where the angle between any two vectors is $\pi/3$.  
Also, all vectors have the same length, except one which has twice the length of the others. 
This basis is not VR but SVR. 
The Gram matrix is:
\begin{equation}
\label{eq_Dn_const_A}
\Gamma_{D_{n}}^{(1)}=
\left(
\begin{array}{cccccccc}
4 & 2 &2 &. . . &2 \\
2 &2& 1& . . . & 1 \\
2 &1& 2&  . . . & 1 \\
. & . & .& . .  . & . \\
2 & 1 & 1  & . . . & 2
\end{array}
\right).
\end{equation}

\begin{theorem}
\label{theo_nbReg_Lin_Dn_const_A}
Consider a $D_n$-lattice basis defined by the Gram matrix~\eqref{eq_Dn_const_A}.
The decision boundary function $f$, defined on $\D(\B)$, has a number of affine pieces equal to
\small
\begin{equation}
\label{eq_Dn_constA}
 \sum_{i=0}^{n-2} \underset{(l)}{ \underbrace{ \left((n-1 -i)+\binom{n-1-i}{2}\right)}} \times \underset{(o)}{\underbrace{\binom{n-1}{i}}}.
\end{equation}
\normalsize
\vspace{-2mm}
\end{theorem}

\vspace{-4mm}
\textbf{Sketch of proof}  We briefly explain what are the $(l)$ and $(o)$ terms in \eqref{eq_Dn_constA}.
On Figure~\ref{fig_neighbor_D3_const_A},  $i$-simplices are illustrated. Any $i$-simplex is defined by a point $x \in \CBONE$ and 
the $i$ points taken from $\CN(x) \cap \CBZERO$.
Any piece of $f$, depicted in Figure~\ref{fig_func_D3_const_A}, is also a piece of the decision boundary of one of the $i$-simplices: i.e. this latter boundary is a function separating the only corner of the simplex $x \in \CBONE$ from the other corners in $\CN(x) \cap \CBZERO$. 
Hence, each distinct $i$-simplex generates $i$ pieces in $f$.  
The number of pieces of $f$ is then obtained by finding the number of $i$-simplices: e.g. on Figure~\ref{fig_neighbor_D3_const_A}, 
there are two 1-simplices and one 3-simplex, thus $f$ has 5 pieces. 
In \eqref{eq_Dn_constA}, $(l)$ represents the dimension of a given simplex and $(o)$ the number of such simplices in $\PPB$.
$(l)$ and $(o)$ are found by exploiting the structure of $\PPB$ as done in Appendix~\ref{App_Dn_first_kind}. $\blacksquare$ \\

\begin{figure}[t]
\centering
\vspace{-8mm}
\begin{minipage}{.45\textwidth}
  \centering
  \vspace{-8mm}
  \includegraphics[width=.34\linewidth]{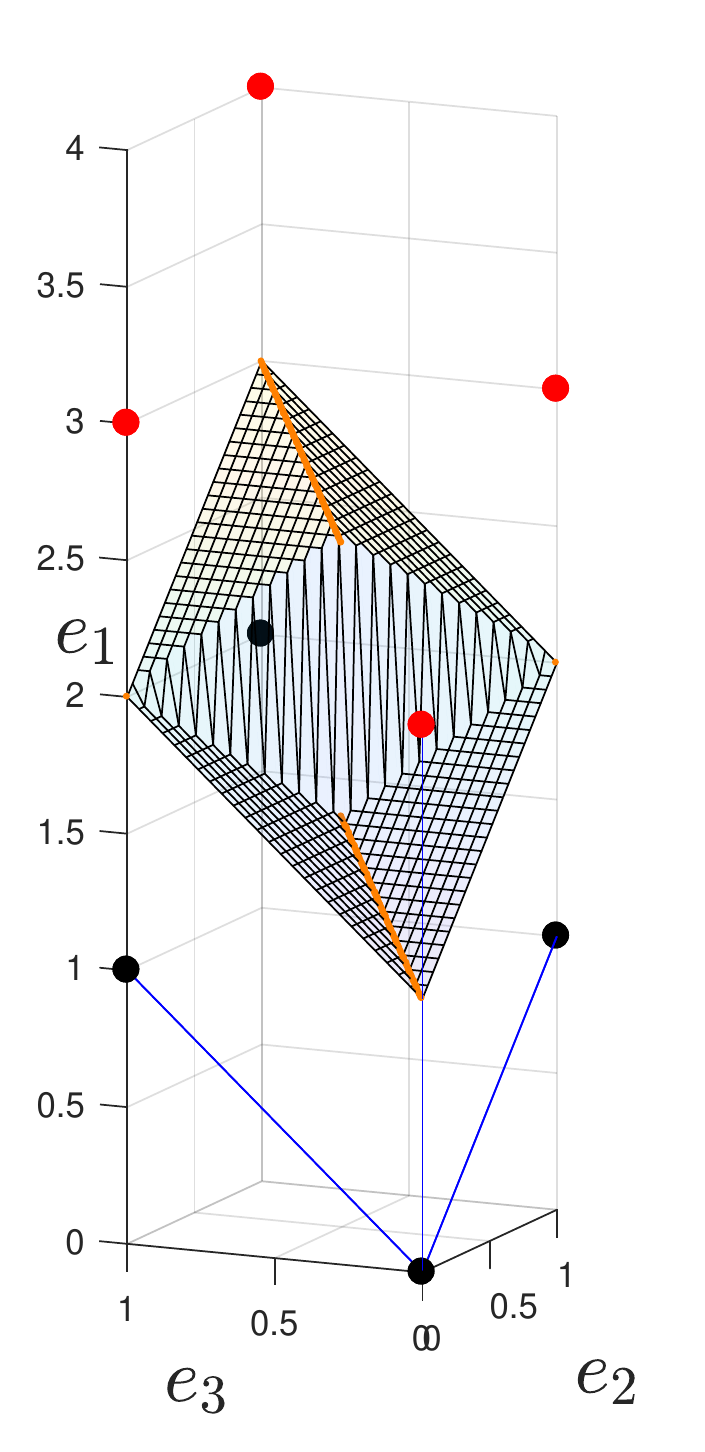}
  \caption{CPWL decision boundary function for $D_3$ defined by the basis of Construction A. 
	     The basis is rotated: $b_1$ is collinear with $e_1$.}
  \label{fig_func_D3_const_A}
\end{minipage}%
\hspace{8mm}
\begin{minipage}{.45\textwidth}
  \centering
  \includegraphics[width=.34\linewidth]{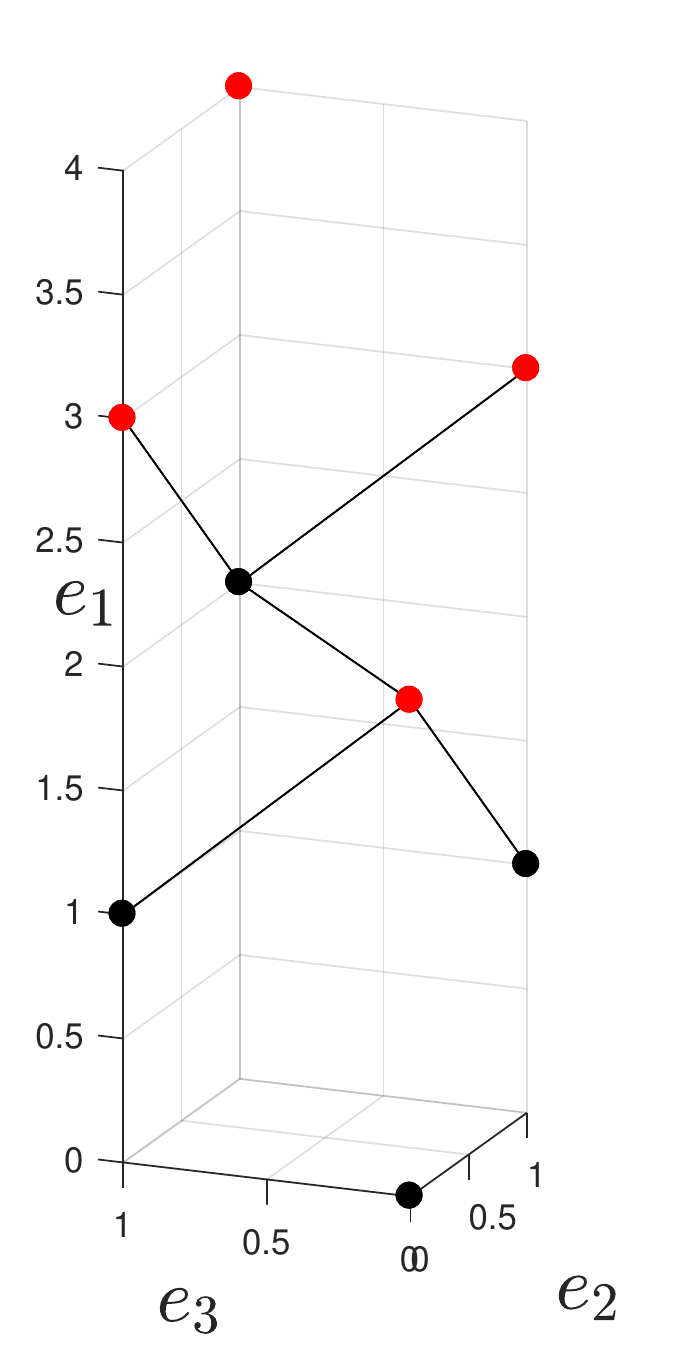}
  \caption{``Neighbor" figure of $\mathcal{C}_{\PPB}$ for $D_3$ defined by the basis of Construction A. 
	        Each edge connects a point $x \in \CBONE$ to an element of $\CN(x) \cap \CBZERO$.
	        The $i$ edges connected to a point $x \in \CBONE$ are $1$-faces of a $i$-simplex.}
  \label{fig_neighbor_D3_const_A}
\end{minipage}
\end{figure}


Given the basis orientation as in~Theorem~\ref{th_func_VR},
the projection of $b_j$ on $\mathcal{D}(\B)$ is $b_j$ itself, $2 \le j \le n$. 
We also denote the bisector hyperplane
between two vectors $b_j, b_k$ by $BH(b_j, b_k)$
and its normal vector is taken to be $v_{j,k}=~b_j-b_k$. 
We define the folding transformation $F:\D(\B) \rightarrow \D'(\B)$ as follows.  
Let $\tilde{y} \in \D(\B)$, for all $2\le j < k \le~n$,
compute $\tilde{y} \cdot v_{j,k}$ (the first coordinate of $v_{j,k}$ is zero).
If the scalar product is non-positive, replace $\tilde{y}$
by its mirror image with respect to $BH(b_j, b_k)$.
There exist $\binom{n-1}{2}$ hyperplanes for mirroring.

\begin{theorem}
\label{theo_Dn_const_A_linear}
Let us consider the lattice $D_n$ defined by the Gram matrix~\eqref{eq_Dn_const_A}. 
We have (i) for all $\tilde{y} \in \D(\B)$, $f(\tilde{y}) = f(F(\tilde{y}))$ and (ii) $f$ has exactly 
\vspace{-1.5mm}
\begin{equation}
2n-1
\vspace{-1.5mm}
\end {equation}
pieces on $\D'(\B)$.
This is to be compared with (\ref{eq_Dn_constA}). 
\end{theorem}

The folding procedure is identical to the one used for $A_n$ (see \cite{Corlay2019}): 
the number of pieces to evaluate is reduced to a linear number via $\mathcal{O}(n^2)$ reflections with respect to the bisector hyperplanes
between any pair of vectors in $\B \backslash \{ b_1 \}$. 
The proof presents no novelty and is deferred to Appendix~\ref{App_folding_const_A}.

As a result, $f$ can be computed by a ReLU network of depth $\mathcal{O}(n^2)$ and width $\mathcal{O}(n)$ 
(with the strategy explained in Appendix~\ref{sec_folding_relu}).

\subsection{Second basis of $D_n$}
We investigate a second basis of $D_n$.
All basis vectors have the same length 
but we have both $\pi/3$ and $\pi/2$ angles between the basis vectors. 
This basis is not VR but SVR.
It is defined by the following Gram matrix.
\begin{equation}
\label{eq_second_kind}
\Gamma_{D_{n}}^{(2)}=
\left(
\begin{array}{cccccccc}
2 & 0 &1 &. . . &1 \\
0 &2& 1& . . . & 1 \\
1 &1& 2&  . . . & 1 \\
. & . & .& . .  . & . \\
1 & 1 & 1  & . . . & 2
\end{array}
\right).
\end{equation}

\begin{theorem}
\label{theo_nbReg_Lin_Dn_second_kind}
Consider a $D_n$-lattice basis defined by the Gram matrix~\eqref{eq_second_kind}.
The decision boundary function $f$, defined on $\D(\B)$, has a number of affine pieces equal to
\begin{equation}
\label{eq_nbReg_Dn_second_kind}
\sum_{i=0}^{n-2} \left(   \underset{(l)}{\underbrace{\left[ 1+(n-2-i) \right]}}+  \underset{(ll)}{\underbrace{\left[\underset{(1)}{\underbrace{1+2(n-2-i)}}+ \underset{(2)}{\underbrace{\binom{n-2-i}{2}}} \right]}} \right) 
\times \underset{(o)}{\underbrace{\ \binom{n-2}{i}}}-1.
\end{equation}
\end{theorem}

\begin{figure}[H]
\centering
\vspace{-8mm}
\begin{minipage}{.4\textwidth}
    \centering
	\vspace{8mm}
    \includegraphics[scale=0.35]{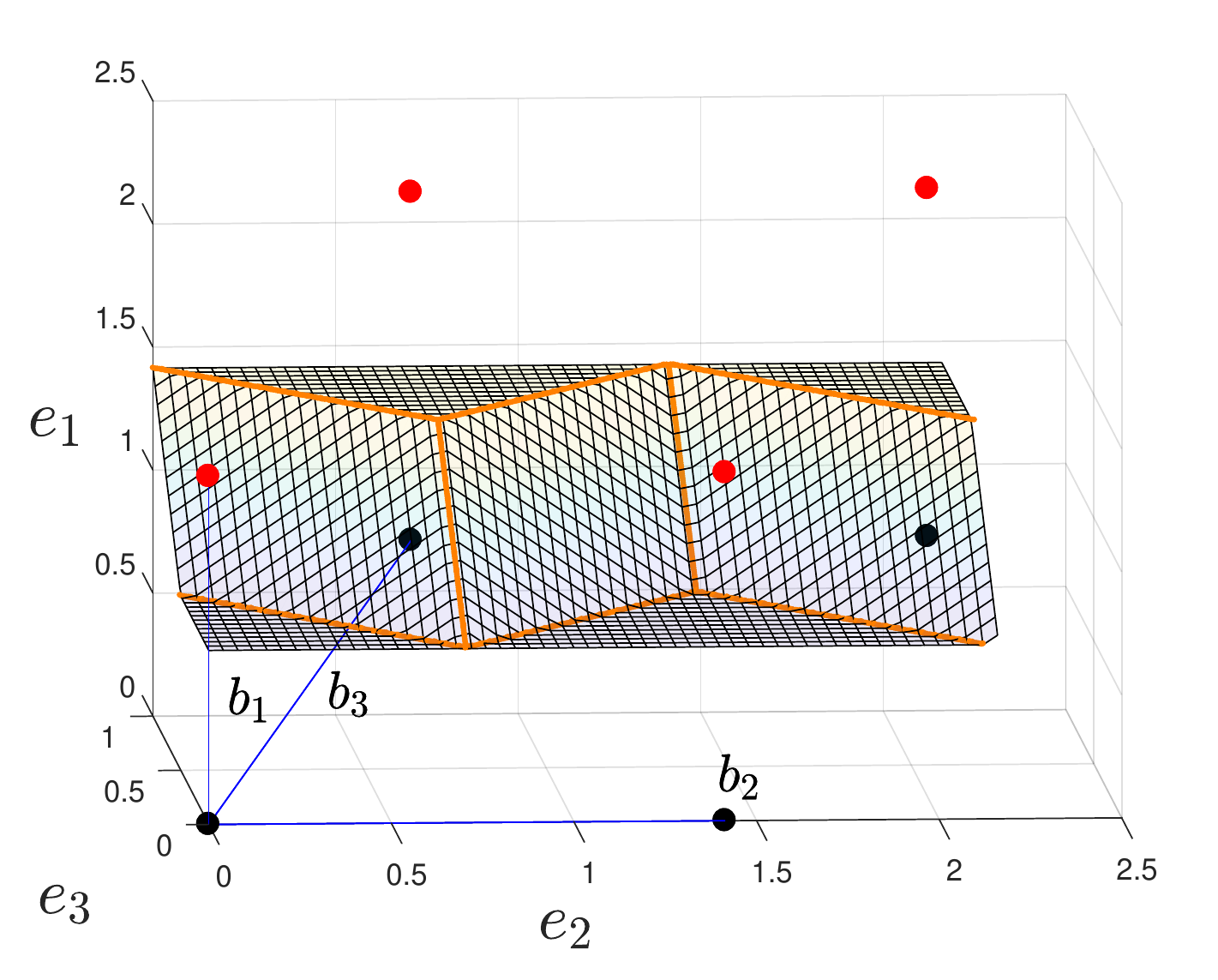}
     \caption{CPWL boundary function for $D_3$ defined by the second basis. 
	        The basis is rotated to better illustrate the symmetry: $b_1$ is collinear with $e_1$.}
     \label{fig_func_D3_second_kind}
\end{minipage}%
\hspace{8mm}
\begin{minipage}{.4\textwidth}
    \centering
    \includegraphics[scale=0.35]{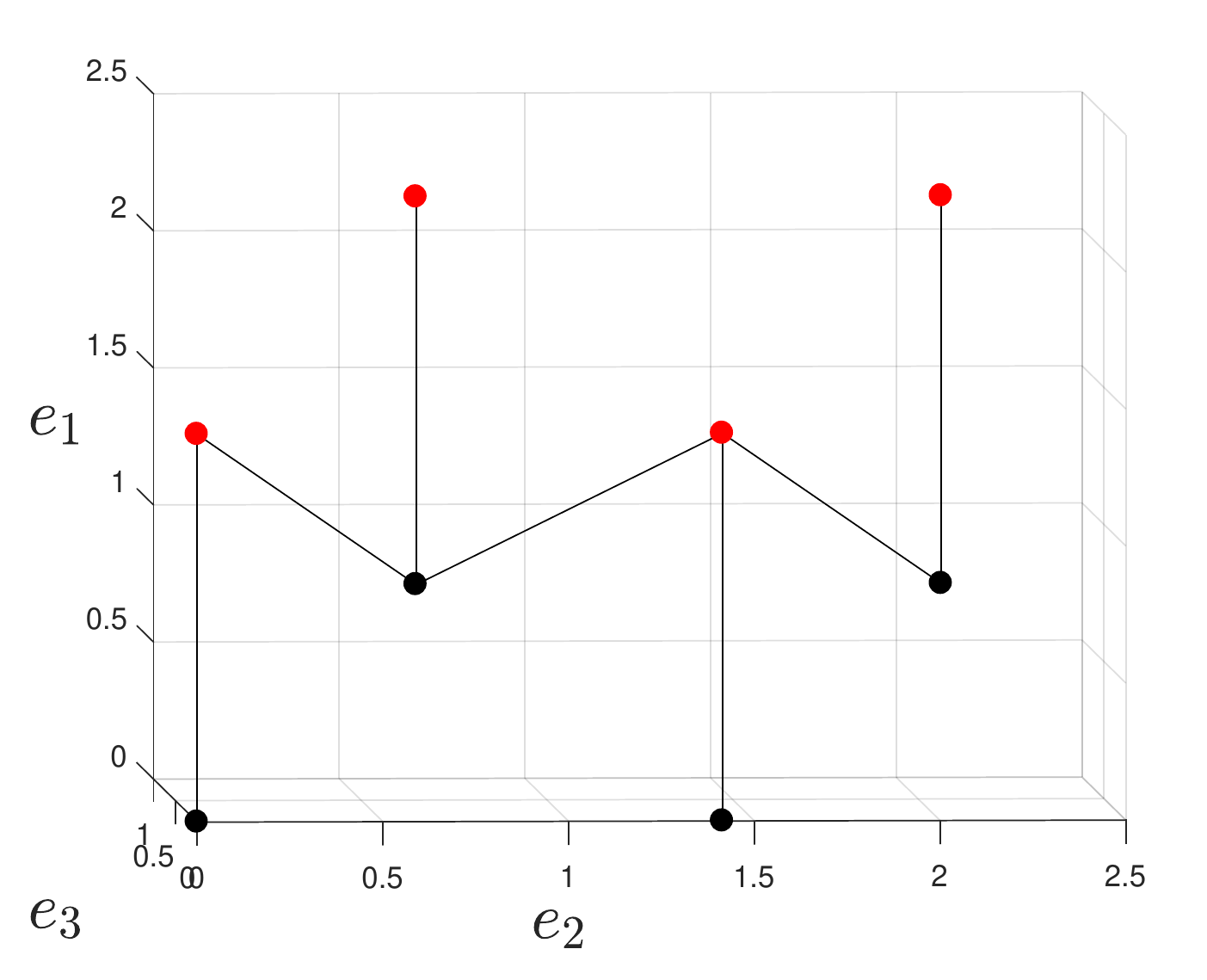}
     \caption{``Neighbor" figure of $\mathcal{C}_{\PPB}$ for $D_3$ defined by the second basis.}
     \label{fig_simplex_D3_2}
\end{minipage}
\vspace{-2mm}
\end{figure}

We give an example to gain insight into the above formula. The proof is deferred to Appendix~\ref{App_Dn_second_kind}.
The previous sketch of proof highlights that we need to count the
$i$-simplices to get the number of pieces of $f$.
This is achieved by finding the different ``neighborhood patterns" (this gives $(l)$ and $(ll)$) and 
counting the occurrence of $i$-simplices for each of these patterns (this gives $(o)$).
The following example presents the two different patterns encountered with this basis of $D_n$.
In the Appendix, we then count the number of simplices  (i.e. $(o)$) in each of these two categories.

\begin{example}
\label{ex_second_kind}
Consider the decision boundary function of Figure \ref{fig_func_D3_second_kind}. 
We are looking for the different ``neighborhood patterns" by studying Figure~\ref{fig_simplex_D3_2}: 
i.e. we are looking for the different ways to find the neighbors of $x \in \CBONE$ in $  \CN(x) \cap \CBZERO$, depending on the form of $x$.
In the sequel, $(l)$, $(ll)$, and $(1)$, $(2)$ refer to Equation~\eqref{eq_nbReg_Dn_second_kind}
and $\sum_j b_j$ denotes any sum of points in the set $\{0,b_j\}_{j=3}^{n}$. 
We recall that adding $b_1$ to any point $x \in \CBZERO$ leads to a point in $\CBONE$.

$(l)$ Firstly, we consider any point in $\CBONE$ of the form $\sum_j b_j + b_1$.
Its neighbors in $\CBZERO$ are $\sum_j b_j$ and any $\sum_j b_j+b_i$, 
where $b_i$ is any basis vector having an angle of $\pi/3$ with $b_1$ such that $\sum_j b_j+b_i$ is not outside $\PPB$.
For $n=3$, the closest neighbors of $0+b_1$  in $\CBZERO$ are $0$ and $b_3$. 
$b_2$ is perpendicular to $b_1$ and is not a closest neighbor of $b_1$. We get a $1+n-2$-simplex generating 2 pieces $f$.
The point $b_3+b_1$ also belongs to this category except that no basis vectors having an angle of $\pi/3$ with $b_1$ can be added to $b_3$ without leaving $\PPB$
(i.e $b_3+b_3+b_1$ is not in $\PPB$). 
Hence, the only closest neighbor of $b_3+b_1$ in $\CBZERO$ is $b_3$: we have a $1+n-2-1$-simplex.
Note that this pattern is the same as the (only) one encountered for $A_n$ with the basis given by Equation~$\eqref{eq_basis_An}$ (see Appendix~\ref{App_theo4}). 

$(ll)$ The second pattern is obtained with any point of the form $\sum_j b_j + b_2+b_1$ and its neighbors in $\CBZERO$, 
where $b_2$ is the basis vector orthogonal to $b_1$.
$\sum_j b_j + b_2$ and any $\sum_j b_j + b_2 +b_i$, $\sum_j b_j+b_k$ are neighbors of this point in $\CBZERO$, where $b_i$, $b_k$ are any basis vector having an angle of $\pi/3$ 
with $b_1$ such that (respectively) $\sum_j b_j + b_2 +b_i$, $\sum_j b_j+b_k$  are not outside $\PPB$.
For $n=3$, the closest neighbors of $0+b_2+b_1$  in $\CBZERO$ are $b_2$, $b_2 + b_3$, and
$b_3$. We get a $1+2 (n-2)$~-~simplex.
For $b_3+b_2+b_1$, it is the same pattern, except that in this case no $b_i$ can be added to $b_3+b_2$ without leaving $\PPB$: we have a $1+2 (n-2-1)$-simplex.
These terms generate $(1)$ in the formula.
Moreover, for $n=3$ one ``neighborhood case" is not happening: 
from $n=4$, the points $b_i+b_j \in \CBZERO$, $3 \leq i < j \leq n $, are also closest neighbors of $b_2 + b_1$. 
This~explains~the~binomial~coefficient $(2)$.
\end{example}

Given the basis orientation as in Theorem~\ref{th_func_VR},
the folding transformation
$F:\D(\B) \rightarrow \D'(\B)$ is defined as follows.
Let $\tilde{y} \in \D(\B)$, for all $3\le j < k \le~n$,
compute $\tilde{y} \cdot v_{j,k}$ (the first coordinate of $v_{j,k}$ is zero).
If the scalar product is non-positive, replace $\tilde{y}$
by its mirror image with respect to $BH(b_j, b_k)$.
There exist $\binom{n-2}{2}$ hyperplanes for mirroring.

\begin{theorem}
\label{theo_Dn_second_kind_lin}
Let us consider the lattice $D_n$ defined by the Gram matrix~\eqref{eq_second_kind}. 
We have (i) for all $\tilde{y} \in \D(\B)$, $f(\tilde{y})~=~f(F(\tilde{y}))$ and (ii) $f$ has exactly 
\vspace{-1.5mm}
\begin{equation}
6n-6
\vspace{-1.5mm}
\end {equation}
pieces on $\D'(\B)$.
This is to be compared with \eqref{eq_nbReg_Dn_second_kind}. 
\end{theorem}

\textbf{Sketch of proof} 
\ To count the number of pieces of $f$, defined on $\D'(\B)$, 
we need to enumerate the cases where both $x \in \CBONE$ and $x'\in \CN(x) \cap \CBZERO$ are 
on the non-negative side of all reflection hyperplanes. 
Among the points in $\mathcal{C}_{\PPB}$ only the points
\begin{enumerate}
\item $x_1=b_3+...+b_{i-1}+b_i$ and $x_1+b_1$,
\item $x_2=b_3+...+b_{i-1}+b_i+b_2$ and $x_2+b_1$,
\end{enumerate}
$i \leq n$, are on the non-negative side of all reflection hyperplanes.
Via Example~\ref{ex_second_kind}, it is then easily seen that the number of pieces of $f$, defined on $\D'(\B)$,
is given by equation~\eqref{eq_nbReg_Dn_second_kind} reduced as follows: 
the three terms $(n-2-i)$ (i.e. $2(n-2-i)$ counts for two), the term $\binom{n-2-i}{2}$, and the term $\binom{n-2}{i}$ become 1 at each step $i$,
for all $0 \leq i \leq n-3$ (except $\binom{n-2-i}{2}$ which is equal to 0 for $i=n-3$).
Hence, \eqref{eq_nbReg_Dn_second_kind} becomes $(n-3)\times(2+4)+(2+3)+1$, which gives the announced result. $\blacksquare$ \\

Consequently, $f$ can be computed by a ReLU network of depth $\mathcal{O}(n^2)$ and width $\mathcal{O}(n)$ 
(with the strategy explained in Appendix~\ref{sec_folding_relu}).

\subsection{$E_n$}

Finally, we investigate $E_n$, $6 \le n \le 8$. 
$E_8$ is one of the most famous and remarkable lattices due to its exceptional density relatively to its dimension (it was recently proved that $E_{8}$ 
is the densest packing of congruent spheres in 8-dimensions (\cite{Viazovska2017})). 
The basis we consider is almost identical to the basis of $D_n$ given by \eqref{eq_second_kind}, except one main difference:
there are two basis vectors orthogonal to $b_1$  instead of one. 
This basis is not VR but SVR.
It is defined by the following Gram matrix.

\begin{equation}
\label{eq_En}
\Gamma_{E_{n}}=
\left(
\begin{array}{cccccccc}
2 & 0 &0 &1&. . . &1 \\
0 &2& 1&1& . . . & 1 \\
0 &1& 2&1&  . . . & 1 \\
1 &1&1&2& . . .  &1 \\ 
. & . & .& . & . .  . & . \\
1 & 1 & 1 &1 & . . . & 2
\end{array}
\right).
\end{equation}

\begin{theorem}
\label{theo_nbReg_Lin_En}
Consider an $E_n$-lattice basis, $6 \le n \le 8$, defined by the Gram matrix~\eqref{eq_second_kind}.
The decision boundary function $f$, defined on $\D(\B)$, has a number of affine pieces equal to
\tiny
\begin{equation}
\label{eq_En}
\sum_{i=0}^{n-3} \left( \underset{(l)}{\underbrace{\left[ 1+(n-3-i) \right]}} +\underset{(ll)}{\underbrace{2 \left[ 1+ 2(n-3-i)+ \binom{n-3-i}{2} \right]}} + \underset{(lll)}{\underbrace{\left[ \underset{(1)}{\underbrace{1+3(n-3-i) }}+  \underset{(2)}{\underbrace{3\binom{n-3-i}{2}}} + \underset{(3)}{\underbrace{\binom{n-3-i}{3}}} \right]} }\right)\underset{(o)}{\underbrace{\binom{n-3}{n-i}}}-3.
\end{equation}
\normalsize
\end{theorem}
\vspace{-1mm}
\textbf{Sketch of proof}  We first highlight the similarities with the function of $D_n$ defined by~\eqref{eq_second_kind} (we use the same numbering as in the in Example~\ref{ex_second_kind}).  As with $D_n$, we have case $(l)$. Case $(ll)$ of $D_n$ is also present but obtained twice because of the two orthogonal vectors.  The terms $n-2-i$ in $(l)$ and $(ll)$ of Equation~\eqref{eq_nbReg_Dn_second_kind} are replaced by $n-3-i$ also because of the additional orthogonal vector.

There is a new pattern $(lll)$: any point of the form $\sum_j b_j + b_3+b_2+b_1$ and its neighbors in $\CBZERO$, where $\sum_j b_j$ represents any sum of points in the set $\{0,b_j\}_{j=4}^{n}$.
For instance, the closest neighbors in $\CBZERO$ of $b_3 + b_2+b_1 \in \CBONE$  are the following points, which we can sort in three groups as on Equation~\eqref{eq_En}: 
(1) $b_2+b_j$, $b_3+b_j$, $b_2+b_3+b_j$, (2) $b_j+b_k$, $b_2+b_j+b_k$, $b_3+b_j+b_k$, (3) $b_j+b_i+b_k$,
 $4 \leq i<j<k \leq n$. The formal proof is available in Appendix~\ref{App_func_En}. $\blacksquare$ \\

Given the basis orientation as in Theorem~\ref{th_func_VR},
the folding transformation
$F:\D(\B) \rightarrow \D'(\B)$ is defined as follows.
Let $\tilde{y} \in \D(\B)$, for all $4\le j < k \le~n$ and $j=2,k=3$,
compute $\tilde{y} \cdot v_{j,k}$ (the first coordinate of $v_{j,k}$ is zero).
If the scalar product is non-positive, replace $\tilde{y}$
by its mirror image with respect to $BH(b_j, b_k)$.
There exist $\binom{n-3}{2}$+1 hyperplanes for mirroring.
We get the following theorem, whose proof is available in Appendix~\ref{App_folding_En}.

\vspace{-1mm}
\begin{theorem}
\label{theo_En_folding}
Let us consider the lattice $E_n$, $6\leq n \leq 8$, defined by the Gram matrix~\eqref{theo_nbReg_Lin_Dn_second_kind}. 
We have (i) for all $\tilde{y} \in \D(\B)$, $f(\tilde{y})~=~f(F(\tilde{y}))$ and (ii) $f$ has exactly 
\vspace{-1.5mm}
\begin{equation}
12n-40
\vspace{-1.5mm}
\end {equation}
pieces on $\D'(\B)$.
This is to be compared with~\eqref{eq_En}. 
\end{theorem}

\vspace{-1mm}

Consequently, $f$ can be computed by a ReLU network of depth $\mathcal{O}(n^2)$ and width $\mathcal{O}(n)$ 
(with the strategy explained in Appendix~\ref{sec_folding_relu}).
\clearpage


\bibliography{bibli}

\clearpage


\appendix

\section{Recent results on the expressive power of deep neural networks}
\label{App_expre_power_NN}

The ultimate goal of research on the expressive power of deep neural networks 
is to find a large function class that can only be addressed via deep neural networks
and no other ways, including shallow networks and ``conventional approaches" (i.e. not deep neural networks).
Results of research works in this field are usually either capacity bounds (i.e. what can do a deep neural network) or separation bounds.
These bounds can depend on (i) the approximation error, (ii) the dimension of the input
as well as (iii) the width and (iv) the depth of the neural network. 

Unfortunately, results on larger function class tend to be looser as the bounds
have to hold for the worst-case scenario. 
Moreover, one of the (empirically observed) strength of neural networks compared to other techniques is 
their ability to efficiently approximate a given function.
Therefore, stronger theorems can be obtained for specific functions
but are less representative. 

Consequently, papers in the literature can be sorted based on the ``size" of the function class addressed 
and whether or not the results depend on (i),(ii),(iii), and (iv).
The present work addresses a small function class
(even though it may be a starting point to study algebraic functions), (i), (ii), (iii) and (iv).
The following list is not exhaustive and does not include older results related to the field of circuit complexity.

\cite{Eldan2016} proved a separation theorem including (i), (ii), and (iii) for a large class of function, 
namely ``radial" functions. Nevertheless, this separation holds only for two-layer and three-layer neural networks, thus (iv) is missing.
Also, note that \cite{Daniely2017} found a simpler proof of this result and  \cite{Safran2017} extended this separation result between two-layer and three-layer network to a larger class of function including the Euclidean unit ball. 

\cite{Montufar2014} achieved the best capacity theorem for deep ReLU neural networks including (ii), (iii), and (iv).
Similarly to our work, this is achieved via a small function class.
As shown in the Appendix of \cite{Corlay2019}, these functions 
can be computed via conventional methods 
as they are based on a periodic one dimensional function.

\cite{Telgarsky2016} proved a separation theorem between shallow and 
deep networks (this separation theorem was improved by \cite{Arora2018} by re-using the same ideas) including (i), (iii) and (iv).
Since this theorem is based on a one dimensional triangle wave function (see Appendix~\ref{App_Telgarsky}), (ii) is missing
 (a multi-dimensional function is considered but the bound does not depend on (ii)).

\cite{Arora2018} achieved a multi-dimensional construction with an exponential number of
linear regions requiring only a polynomial number of parameters (part (i) of Theorem~3.9 in the paper) but the proof is based on the fact that 
the high dimensional part of this function can be computed by a conventional method (i.e. the function with ${w}^n$ pieces considered can be computed via a $w$ 2-max, as shown in the proof of Lemma~3.7).

\cite{Raghu2016} showed that any random deep ReLU network achieves an exponential number of linear region depending on (ii),(iii) and (iv).
Additionally, via  the trajectory  length, they
observed that  most of the  random linear regions in  trained networks
are    in   fact    noise   that    should   be    addressed   through
regularization.   

Finally, \cite{Poggio2017} and \cite{Petersen2018} are recent results addressing large function class.


\section{The triangle wave function of  \cite{Telgarsky2016}}
\label{App_Telgarsky}

Telgarsky considers a one dimensional triangle wave function. 
The key observation is that adding two (shifted) copies of a triangle wave function increases 
the number of pieces in an additive manner, while composition acts multiplicatively.
Within a neural network, increasing the width of a layer is equivalent to adding functions, 
while increasing the depth is equivalent to composing functions. 
Hence, a function computed by a deep network, say $f: \R \rightarrow \R $, can have many more oscillations than 
functions computed by networks with few layers, say $g:  \R \rightarrow \R$.
Roughly speaking, if the activation function in each neuron is a triangle wave function with $p$ pieces,
a two-layer $w$-wide network leads to a triangle wave function of $wp$ pieces while a $L$ layers network 
with $\mathcal{O}(1)$-width leads to $p^L$ pieces.  

The difference (or ``error") 
between $f$ and a line can be characterized 
by the triangle areas illustrated on Figure~\ref{fig_func_Tel}.
Hence, the $L^1$ error between $f$ and $g$ is then bounded from below after
summing the triangle areas above the line (resp. below the line)
whenever $g$ is below (resp. above) this same line. 
Indeed, since $g$ has a number of pieces inferior to $f$,
it can only cross this line a limited number of times compared to $f$.

\begin{figure}[H]
    \centering
    \includegraphics[scale=0.8]{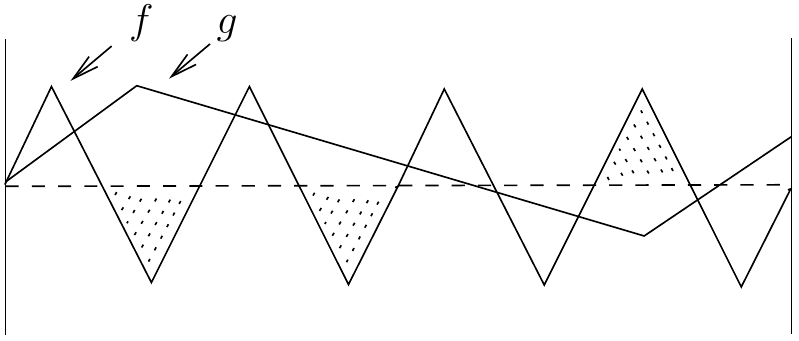}
     \caption{Triangle wave function considered by Telgarsky. 
		The doted triangle areas are used to get a lower bound of the error between $f$ and $g$.}
     \label{fig_func_Tel}
\end{figure}
 
This one-dimensional result is then extended to the $n$-dimensional case in the following manner.
A function $p_{\tilde{y}}(y_1) = ( y_1, \tilde{y})$ is defined. 
$\tilde{y}$ can be understood as an offset. 
The network is then only applied on $y_1$ but the error averaged in the cube $[0,1]^n$.

\begin{figure}[H]
    \centering
    \includegraphics[scale=1]{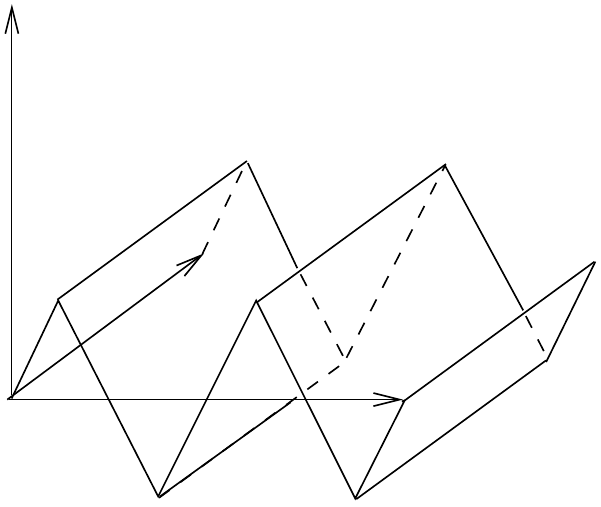}
     \caption{Triangle wave function with offset in $\R^3$. 
		    The number of pieces is not increased compared to the baseline function in $\R^2$.
		     There is no dimensional dependence.}
     \label{fig_func_Tel_2}
\end{figure}

\section{Deferred proofs}



\subsection{Proof of Proposition~\ref{lemma_superexp}: 
			a function with a superexponential number of pieces over a large compact set \label{lem_super_exp}}
\begin{proof}
First, let us define (without loss of generality) the 2-sawtooth ReLU activation function
as $ReLU(u)=u \mod 1$, $\forall \ u \in [0,2]$. 
This function allows to divide any interval into
two equal sub-intervals and then translates the point near the origin.
For illustration in $\R^2$,
as shown in 
Figures~\ref{fig_transla}\&\ref{fig_parti},
$(y_1, y_2)$ is multiplied by $G^{-1}$,
the 2-sawtooth ReLU is applied twice (on each coordinate), 
the output is subtracted from the other output to implement the floor operation,
and
then the result
is multiplied again by $G$. 
This corresponds to partitioning $\mathcal{P}(2\B)$
into four equal regions
$\{ \mathcal{P}(\B), \mathcal{P}(\B)+b_1, \mathcal{P}(\B)+b_2, \mathcal{P}(\B)+b_1+b_2 \}$.

\begin{figure}[H]
\centering
\begin{minipage}{.4\textwidth}
  \centering
  \includegraphics[width=.8\linewidth]{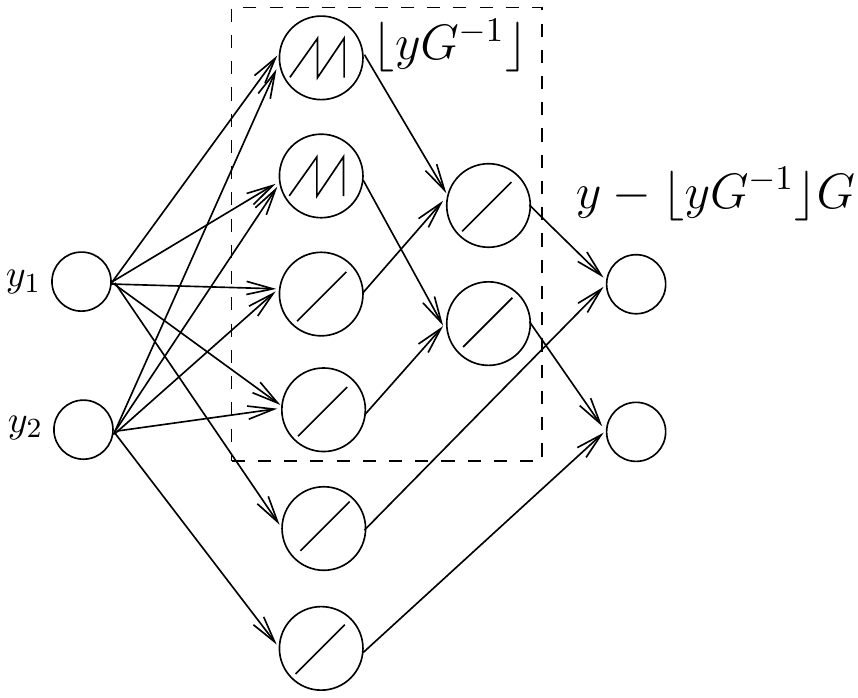}
  \caption{Translation block.}
  \label{fig_transla}
\end{minipage}%
\hspace{8mm}
\begin{minipage}{.4\textwidth}
  \centering
   \vspace{10mm}
   \includegraphics[width=.8\linewidth]{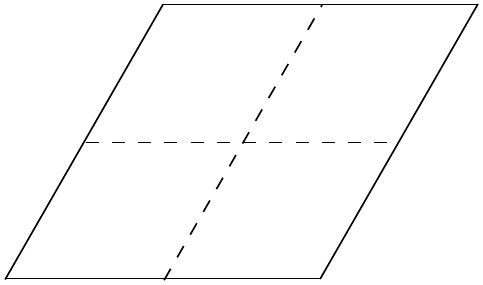}
  \caption{Partition of $\mathcal{P}(2\B)$ induced by one translation block.}
  \label{fig_parti}
\end{minipage}
\end{figure}

In $\R^n$, the 2-sawtooth ReLU is used to partition and translate
$\mathcal{P}(\alpha\B)$, where $\alpha=2^M$ and $M\ge 1$ is an integer.
At step $\ell$, $\ell=1 \ldots M$, a translation block similar to Figure~\ref{fig_transla}
executes the three operations: multiply by $G/2^{M-\ell}$, apply $n$ times a 2-sawtooth ReLU,
finally multiply by $2^{M-\ell}G$. $\mathcal{P}(\alpha\B)$ has $\alpha^n=2^{Mn}$ regions equivalent
to $\PPB$. Similarly, if we consider the set $\mathcal{P}(\{  b_1,  \alpha  b_2, \alpha  b_3,  \ldots,  \alpha
b_n\})$, there are $2^{M(n-1)}$ regions equivalent to $\PPB$ and 
the extended decision boundary function defined on the domain
$\D(\{  b_1,  \alpha  b_2, \alpha  b_3,  \ldots,  \alpha
b_n\})$ has $\Omega(2^{M(n-1)})$ pieces.

Hence, the extended boundary function is computed via two neural networks: 
a first neural network with $3M$ layers based on $M$ translation blocks
of maximum width $3(n-1)$
converts $y_0 \in \mathcal{P}(\{  b_1,  \alpha  b_2, \alpha  b_3,  \ldots,  \alpha
b_n\})$ into $y \in \mathcal{P}(\B)$.
Subsequently, the second neural network, evaluating $f$ defined on $\D(\B)$, takes $y$ as its input.
\end{proof}

Note that this result brings very little novelty as it is very similar to the results of \cite{Montufar2014} 
(we established a link between the function they use and Construction A in the Appendix of \cite{Corlay2019}) 
as well as to the one-dimensional composition argument used by \cite{Telgarsky2016} (see Appendix~\ref{App_Telgarsky}).
Moreover, it hardly justifies the superiority of deep neural networks as this operation can be handled 
by any conventional method. 

\subsection{Proof of Theorem~\ref{th_zero_error}}
\label{App_proof_theo_zero_error}

\subsubsection{Proof of 1.}
To prove 1, we compute an upper bound of the $L_1$ difference between $f$ and the function $h_{\Phi}$ 
defined by the hyperplane $\Phi = \{ y \in \R^n : \ y \cdot e_1  = \frac{1}{2}\times(b_1 \cdot e_1) \}$. 
We show that this bound goes to 0 for large $n$.  It is then obvious that an affine function
can be implemented via a one-neuron linear network.
\begin{proof}
Let $\mathcal{S}(\mathcal{B}_s \cup \{0\})$ denote the non-truncated $n$-simplex (illustrated in Figure~\ref{fig_dode} for $n=3$), 
defined by a basis $\mathcal{B}_s$.
The first step is to prove that all polytopes
$\{ y \in \PPB : y_1 \ge f_m(\tilde{y}), \ y \cdot e_1 \le \frac{1}{2}\times(b_1 \cdot e_1) \}$, 
$\{ y \in \PPB : y_1 \le f_m(\tilde{y}), \ y \cdot e_1 \ge \frac{1}{2}\times(b_1 \cdot e_1) \}$
are indeed truncated versions of $\mathcal{S}(\mathcal{B}_s \cup \{0\})$ and that there are $K=2^n$ distinct versions of them in $\PPB$.
We rely essentially on the proof of Theorem~4 in \cite{Corlay2019} (available in Appendix~\ref{App_theo4}): 
this proof shows that there are $\sum_{i=1}^{n} \binom{n-1}{n-i}=2^{n-1}$ distinct convex regions in $f$. 
Since for any convex region there is a corresponding concave region, there are $2^n$ of such polytopes.
This same proof also shows that the facets of any of these polytopes  
(except the facets lying in $\Phi$ or in a facet of $\PPB$) are orthogonal to 1-faces of a regular $i$-simplex 
(this simplex is not $\mathcal{S}(\mathcal{B}_s \cup \{0\} )$), $1\le i \le n$, where all these simplices
have one 1-face collinear with a vector defined by $x$ and $x+b_1$, $x \in \CBZERO$ (see Figure~\ref{fig_simplex_A3_neighbor} for the 3-dimensional case).
Hence, all these polytopes are truncated version of the same part of the Voronoi cell of $A_n$ and thus of $\mathcal{S}(\mathcal{B}_s \cup \{0\} )$.

The second step is to get an upper bound of the volume of each truncated simplex.
Clearly, it is inferior to the volume of the non-truncated regular $n$-simplex $\mathcal{S}(\mathcal{B}_s \cup \{0\})$. 
What is the volume of $\mathcal{S}(\mathcal{B}_s \cup \{0\})$? 
This volume is upper-bounded by $\text{Vol}(\PPB)/n!$ (see Subsection~\ref{app_vol_simp} below).

Finally, the $L_1$ distance between $f$ and $h_{\Phi}$ is bounded from above by the sum of the volumes of $K$ $\mathcal{S}(\mathcal{B}_s \cup \{0\})$.
If we take $\text{Vol}(\PPB)=1$, we get 

\begin{equation}
\label{equ_bound}
\int_{\D} |f(\tilde{y})-h_{\Phi}(\tilde{y})|d\tilde{y} < 2^n \cdot \frac{\text{Vol}(\PPB)}{n!} \sim \frac{1}{\sqrt{2 \pi n} 2^{n \log_2(n/e)-n}}, 
\end{equation}
where we used Stirling's approximation.



\end{proof}

\subsubsection{Proof of 2.}
\begin{proof}
We begin with the first part of the second result.
If the compact set $\mathcal{P}(\{ b_1, \alpha b_2, \alpha b_3, \ldots, \alpha b_n\})$,
where $\alpha=2^M$, is large enough, we can make the following approximation:
there are roughly as many Voronoi cell as parallelotopes $\PPB$ in $\mathcal{P}(\{ b_1, \alpha b_2, \alpha b_3, \ldots, \alpha b_n\})$.
This implies that the extended decision boundary $f$ ``contains" at least one non-truncated simplex for each $\PPB$.
With Proposition~\ref{lemma_superexp}, we get that there are $2^{M(n-1)}$ $\PPB$ in $\mathcal{P}(\{ b_1, \alpha b_2, \alpha b_3, \ldots, \alpha b_n\})$.
The volume of one non-truncated simplex is $\Omega \left(1/n^n\right)$ for an edge length of $\sqrt{2}$ (see Subsection~\ref{app_vol_simp} below).
Hence, if $K$ is the number of $\PPB$ in $\mathcal{P}(\{ b_1, \alpha b_2, \alpha b_3, \ldots, \alpha b_n\})$, the error between $f$ and $g$ is bounded from below by
\begin{equation}
\int_{\D} |f(\tilde{y})-g(\tilde{y})|d\tilde{y} = K \Omega \left(1/n^n\right),
\end{equation}
where $K = 2^{M(n-1)} = 2^{M(n-1)-n\log_2 (n)}\cdot 2^{n\log_2 (n)}$. 

Similarly to the strategy of Telgarsky (see Appendix~\ref{App_Telgarsky}), 
we can assume that each additional piece in $g$ cancels (at most) the volume of $\mathcal{O}(1)$ simplices in the bound.
Moreover, via Theorem~1 of \cite{Raghu2016} we know that no $L$-deep $w$-wide 
ReLU network with input in $\R^{n-1}$ can compute more than $\mathcal{O}(2^{(n-1)L\log_2(w)})$ pieces.
Consequently, the approximation error is bounded from below by
\begin{equation} 
a \times 2^{(n-1)(M-\log_2(n))-\log_2 (n)}-b \times 2^{(n-1)L\log_2(w)}, 
\end{equation}
where $a$ and $b$ are some constants.
As a result, if we choose $M \ge L \log_2(w)+n$,  then the approximation error is $\Omega \left(2^{(n-1)M-n\log_2(n)}\right)$.

The second part of the result is a direct consequence of Proposition~\ref{lemma_superexp}, 
where the part of $f$ on $\mathcal{\D}(\B)$ is evaluated as follows:
we implement the $\mathcal{O}(n^2)$ reflections, 
that enable to reduce the number of pieces to compute down to a linear number,
via a ReLU neural network of depth  $\mathcal{O}(n^2)$ and  width $\mathcal{O}(n)$ (see Theorem~5 in \cite{Corlay2019} or Section~\ref{sec_big_section}).
\end{proof}

\subsubsection{Volume of the non-truncated simplex\label{app_vol_simp}}

In this subsection, we show that the volume of the non-truncated simplex 
has a lower bound that behaves as $1/n^n$ and an upper bound given by $\text{Vol}(\PPB)/n!$.

Let $V_n$ be the volume of the non-truncated simplex described in Section~\ref{sec_approx}.
This simplex is equivalent to a hyperpyramid obtained by intersecting $V(0)$ with the hyperplane
$\Phi$ orthogonal to $e_1$ and located at a shift of $\frac{1}{2} b_1\cdot e_1$.
Figure~\ref{fig_casquette} illustrates the volume $V_n$ in pink color.
The blue color represents the regular simplex whose vertices are
$\{ 0, \frac{1}{2}b_1, \frac{1}{2}b_2, \ldots, \frac{1}{2}b_n\}$.

\begin{figure}[H]
\centering
\includegraphics[scale=0.9]{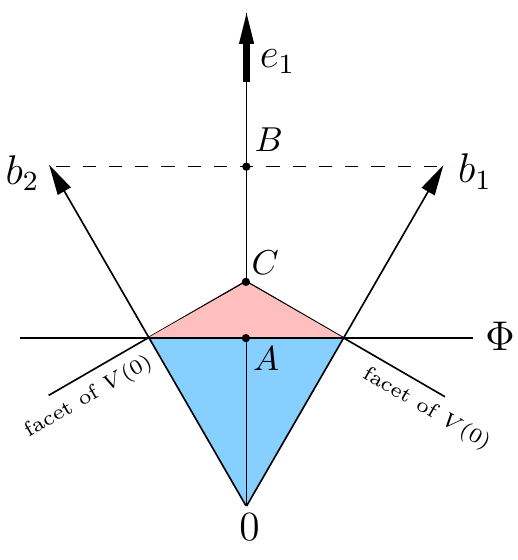}
\caption{Illustration of the non-truncated simplex (in pink on the figure).}
\label{fig_casquette}
\end{figure}

The volume of the hyperpyramid is $V_n=\frac{S \times h}{n}$,
where $S$ is the $n-1$-dimensional volume of this hyperpyramid facet lying on $\Phi$
and $h$ is the hyperpyramid height.

We start by determining $h$. Let $O$ be the point representing the origin in $\R^n$.
Denote by $C$ the centroid of the regular simplex whose vertices are $\{ 0, b_1, b_2, \ldots, b_n\}$.
The line $OC$ cuts $\Phi$ at the point $A$ and the hyperplane $\{b_i\}_{i=1}^n$ at the point $B$.
Then $h=OC-OA$ becomes
\[
h=OC-\frac{1}{2}OB=\frac{n}{n+1}OB-\frac{1}{2}OB=\frac{n-1}{2(n+1)} \sqrt{\frac{n+1}{n}}, 
\]
because $OB=\sqrt{\frac{n+1}{n}}$ is the height of the regular simplex with edge length $\sqrt{2}$.
The area $S$, i.e. the $n-1$-dimensional volume of the facet lying on $\Phi$,
is bounded from below by the area $S'$ of the blue simplex facet lying on $\Phi$.
Figure~\ref{fig_casquette} shows them equal in $\R^2$, but the facet of the pink simplex
will be larger that its blue counterpart for $n \ge 3$. From the formula of a regular simplex
volume, we get
\[
S \ge S'=\frac{a^{n-1}}{(n-1)!} \frac{\sqrt{n}}{2^{(n-1)/2}},~~~a=\frac{1}{2}\|b_1\|=\frac{1}{\sqrt{2}}.
\]
Finally, the lower bound of $V_n$ is
\begin{align*}
V_n \ge \frac{S' \times h}{n} &=\frac{n(n-1)}{2^n\times (n+1)^{3/2} \times n!}\\
&\sim \frac{1}{\sqrt{2\pi}} \frac{1}{(2n/e)^n}.
\end{align*}
Hence, the volume of the non-truncated simplex is $\Omega(1/n^n)$.

Moreover, the vectors defined by $n$ points of the simplex and their intersections generate a parallelotope.
This parallelotope is included in $\PPB$, its volume is thus inferior to the one of $\PPB$.
The volume of any  simplex is $n!$ times smaller than the volume of the parallelotope generated by the $n$ points.
Hence, the volume of the non-truncated simplex is bounded from above
by $\text{Vol}(\PPB)/n!$.

\subsection{Proof of Theorem~\ref{theo_nbReg_Lin_Dn_const_A}: number of pieces of $f$ with the basis of Construction A of $D_n$}
\label{App_Dn_first_kind}
This proof follows the same logic as the proof of the boundary function for $A_n$ (presented in \cite{Corlay2019} and available in Appendix~\ref{App_theo4}).
\begin{proof} We recall that any piece of $f$ is located in a hyperplane orthogonal
to a segment joining a point $x \in \mathcal{C}_{\PPB}^{1}$ and one of its neighbors $x' \in \CN(x) \cap \mathcal{C}_{\PPB}^{0}$.
For a given point in $\CBONE$, the neighbors of interest can be found via the following property of this basis of $D_n$:
\begin{gather}
\label{eq_prop_constA}
\begin{split}
&\forall \ x \in~\CBZERO, \ x' \in D_n \backslash \{b_j,0\} \backslash \{b_j + b_j \},\ 2 \leq i < j \leq n: \\
&\{x+ b_j\} \backslash \{x\} \in \CN(x+b_1),  \ \{x+ b_i+b_j\} \backslash \{x\} \in \CN(x+b_1),  \\ 
&x+ x' \not\in \CN(x+b_1) \cap \CBZERO.
\end{split}
\end{gather}
In other words, the two main differences with $A_n$ are that 
(i) summing two basis vectors $b_i+b_j \in \CBZERO$, $2 \leq i < j \leq n$, 
results in a point which is a closest neighbor of $b_1 \in \CBONE$
and (ii) $x \in \CBZERO$ is not a closest neighbor of $x + b_1$. 
This clearly appears on Figure~\ref{fig_neighbor_D3_const_A}. 
A point $x \in \mathcal{C}^1_{\PPB}$ and its neighbors $\CN(x)~\cap~\CBZERO$ form a $|\CN(x) \cap \CBZERO|$-simplex $\mathcal{S}$.

Now, consider the decision boundary function of a $i$-simplex separating the top corner 
(i.e. $\mathcal{C}^1_{\mathcal{S}}= \{ x \}, \, x \in \mathcal{C}^1_{\PPB}$) from all the other corners  (i.e. $\mathcal{C}^0_{\mathcal{S}}=\CN(x) \cap \CBZERO$). 
As long as no corner in $\mathcal{C}^0_{\mathcal{S}}$ has its first coordinate $e_1$ larger than the first coordinate
of the top corner, which is always the case with the orientation of the basis as in~Theorem~\ref{th_func_VR},
the function is convex and has $i$ pieces. 
The maximal size of such simplex in $\PPB$ is obtained
by taking the point $b_1$,  the $n-1$ points $b_j$, $2 \le j \le n$, 
and the $\binom{n-1}{2}$ points $b_i+b_j$, $2 \le i < j \le n$.
Hence, the decision boundary function $f$ has a number of affine pieces equal to
\begin{align}
\sum_{i=1}^{n-1 + \binom{n-1}{2}}i \times ( \text{$\#$ $i$-simplices}),
\end{align}
where, for each $i$-simplex, only one corner $x$ belongs to $\CBONE$
and the other corners constitute the set $\CN(x) \cap \CBZERO$.

We now count the number of $i$-simplices.
We walk in $\CBZERO$ and for each of the $2^{n-1}$ points $x \in \CBZERO$ 
we investigate the size of the simplex where the top corner is $x + b_1 \in \CBONE$. 
This is achieved by counting the number of elements in $\CN(x+~b_1)~\cap~\CBZERO$, 
via the property given by (\ref{eq_prop_constA}). 
Starting from the origin, one can form a $n-1 + \binom{n-1}{2}$-simplex with the point $b_1$,  the $n-1$ points $b_j$, $2 \le j \le n$, 
and the $\binom{n-1}{2}$ points $b_i+b_j$, $2 \le i < j \le n$. 
Then, from any $b_{j_1}$, $2 \leq j_1\leq n$, one can only use the $n-1$ remaining basis vectors to generate a simplex in $\PPB$. 
Indeed, if we add again~$b_{j_1}$, 
the resulting point  (i.e. the point $b_{j_1}+b_{j_1}$, which is a neighbor of $b_{j_1}+b_1$) is outside $\PPB$ and should therefore not be considered. 
Hence, we get a $(n-1)-1 + \binom{n-1 -1}{2}$-simplex and there are $\binom{n-1}{1}$
ways to choose $b_{j_1}$: any basis vector except~$b_1$. 
Similarly, if one starts the simplex from $b_{j_1}+b_{j_2}$, one can form a $(n-1)-2 + \binom{n-1-2}{2}$-simplex in $\PPB$ and there are $\binom{n-1}{2}$ ways to choose $b_{j_1}+b_{j_2}$. 
In general, there are $\binom{n-1}{i}$ ways to form a $n-1-i + \binom{n-1-i}{2}$-simplex.
\end{proof}

\subsection{Proof of Theorem~\ref{theo_Dn_const_A_linear}: folding of $f$ with the basis of Construction $A$ of $D_n$}
\label{App_folding_const_A}


\begin{lemma}
\label{lem_fold}
Among the elements of $\mathcal{C}_{\PPB}$, only the points
of the form $x=b_2+b_3+...+b_{i-1}+b_i$ and  $x+b_1$, $i \leq n$, 
are on the non-negative side of all $BH(b_j,b_k)$, $2\le  j < k \le~n$.  
\end{lemma}

\begin{proof}
In the sequel, $\sum_i b_i$ denotes any sum of points in the set $\{0,b_i\}_{i=2}^{n}$.
First, consider a point of the form $b_2+b_3+...+b_{j-1}+b_{j+1}+...+b_{i-1}+b_i$, $j+1<i-1\le n-1$.
This point is on the negative side of all $BH(b_j,b_k)$, $ j  < k \leq i$.
More generally, any point $\sum_i b_i$, where $\sum_i b_i$ includes in the sum $b_k$ but not $b_j$, $j<k \leq n$,
is on the negative side of $BH(b_j,b_k)$.
Hence, the only points in $\CBZERO$ that are on the non-negative side of all hyperplanes have the form $b_2+b_3+...+b_{i-1}+b_i$, $i \leq n$.

Moreover, if $x \in \CBZERO$ is on the negative side of one of the hyperplanes $BH(b_j,b_k)$, $2\le  j < k \le~n$, so is $x+b_1$
since $b_1$ is in all $BH(b_j,b_k)$.
\end{proof}

\begin{proof}(of Theorem~\ref{theo_Dn_const_A_linear})
(i) is the  direct result of  the symmetries  in the
$D_n$-lattice  basis where  the  $n-1$ vectors $\{b_j\}_{j=2}^{n}$ form  a regular  $n-1$-dimensional
simplex.  The folding  via $BH(b_j, b_k)$ switches $b_j$  and $b_k$ in
the hyperplane containing $\D(\B)$. $\D(\B)$ is orthogonal to both $e_1$ and $BH(b_j,b_k)$, $2\le  j < k \le~n$. 
Switching $b_j$
and $b_k$ does  not change the decision boundary because  of the basis
symmetry, hence  $f$ is unchanged.

Now, for (ii), how many pieces are left after all reflections?
To count the number of pieces of $f$, defined on $\D'(\B)$, 
we need to enumerate the cases where both $x \in \CBONE$ and $x'\in \CN(x) \cap \CBZERO$ are 
on the non-negative side of all reflection hyperplanes. 
Hence, for any given point $x \in \CBZERO$, that is on the proper side of all reflection hyperplanes, we 
count the number of elements in $\CN(x+b_1)~\cap~\CBZERO$ (via Equation~\eqref{eq_prop_constA}) 
that are also on the proper side of all bisector hyperplanes.


Starting from the origin, due to Lemma~\ref{lem_fold}, one can only form a $2$-simplex with $b_1$, $b_2$, and $b_2+b_3$: 
any other point $b_{j}$, $3 \le j \le n$, is on the negative side  of  $BH(b_2, b_j)$
and any point $b_j+b_k$, $2 \le j < k \le n$, except $b_2+b_3$, 
is on the negative  side  of at least one  $BH(b_i, b_j)$, $i \neq j$.
In general, due to Lemma~\ref{lem_fold}, all points in $\CBONE$ that are on the non-negative side of all hyperplanes, have the form $x=b_2+b_3+...+b_{i-1}+b_{i}+b_1$, $i\leq n$.
There are $n$ of them.
For any $1\leq i \leq n-1$, $x$ has only two neighbors in $\CBZERO$
on the non-negative side of all hyperplanes: $x-b_1$ and $x+b_{i+1}-b_1$
(for $i=n$, $x+b_{i+1}-b_1$ is outside $\PPB$ and $x$ has only one neighbor in $\CBZERO$). 
As a result, $f$, defined on $\D'(\B)$, has $(n-1)\times2+1$ pieces.
\end{proof}

\subsection{Proof of Theorem~\ref{theo_nbReg_Lin_Dn_second_kind}: 
number of pieces of $f$ with the second basis of $D_n$ \label{App_Dn_second_kind}}
\begin{proof}
Similarly to the proof of Theorem~\ref{theo_nbReg_Lin_Dn_const_A}, 
we count the number of simplices.
The number of pieces of $f$ is then obtained by summing 
the number of pieces of the boundary function of each simplex.
 
Hence, we walk in $\CBZERO$ and for each of the $2^{n-1}$ points $x \in \CBZERO$, 
we investigate the size of the simplex where the top corner is $x + b_1 \in \CBONE$. 
This is achieved by counting the number of elements in $\CN(x+~b_1)~\cap~\CBZERO$. 
In this scope, the points in $\CBZERO$ can be sorted into two categories: $(l)$ and $(ll)$. 
In the sequel, $\sum_j b_j$ denotes any sum of points in the set $\{0,b_j\}_{j=3}^{n}$.
These two categories and their properties, illustrated in Example~\ref{ex_second_kind}(see also Equation~\eqref{eq_proof_big} below), are:
\begin{gather}
\label{eq_prop_Dn_2}
\begin{split}
(l) \ &\forall \ x=\sum_j b_j \in~\CBZERO, \ x' \in D_n \backslash \{b_k,0\},\ 3 \leq k\leq n: \\
&x+ b_k \in \CN(x+b_1), \ x+ x' \not\in \CN(x+b_1) \cap \CBZERO. \\
\end{split}
\end{gather}
\begin{gather}
\label{eq_prop_Dn_2__2}
\begin{split}
(ll) \ &\forall \ x= \sum_j b_j + b_2\in~\CBZERO, \ x' \in D_n \backslash \{b_i,-b_2 + b_i,-b_2 + b_i +b_k,0\},\ 3 \leq i < k\leq n: \\
&(1)  \ (a) \ x+ b_i \in \CN(x+b_1), \ (b) \ x -b_2 + b_i \in \CN(x+b_1), \\
&(2) \ x -b_2 + b_i +b_k  \in \CN(x+b_1), \\
&(3) \  x+ x' \not\in \CN(x+b_1) \cap \CBZERO.
\end{split}
\end{gather}
We count the number of $i$-simplices per category.

$(l)$ is like $A_n$ (see Appendix~\ref{App_theo4}). Starting from the origin, one can form a $n-1$-simplex with $0$, $b_1$, and the $n-2$ other basis vectors except $b_2$ (because it is perpendicular to $b_1$). 
Then, from any $b_{j_1}$, $3 \leq j_1\leq n$, one can only add (to $b_{j_1}$) the $n-2$ remaining basis vectors (i.e. neither $b_1$ nor $b_{j_1}$) to generate a simplex in $\PPB$ where the top corner is $b_{j_1}+b_1$. 
Indeed, if we add again~$b_{j_1}$, the resulting point is outside $\PPB$ and should not be considered. 
Hence, we get a $n-2$-simplex and there are $\binom{n-2}{1}$
ways to choose $b_{j_1}$: any basis vector except~$b_1$ and $b_2$. 
Similarly, if one starts the simplex from $b_{j_1}+b_{j_2}$, one can form a $n-3$-simplex in $\PPB$ and there are $\binom{n-2}{2}$ ways to choose $b_{j_1}+b_{j_2}$. 
In general, there are $\binom{n-2}{i}$ ways to form a $n-1-i$-simplex.

$(ll)$ To begin with, we are looking for the neighbors of $b_2+b_1$. First (i.e. property $(1)$), we have the following $1+2 \times (n-2)$  points in $\CN(b_2+b_1) \cap \CBZERO$:
$b_2$, any $b_j+b_2$, $3 \leq j \leq n$, and any $b_j$, $3 \leq j \leq n$.  
Second (i.e. property $(2)$), the $\binom{n-2}{2}$ points $b_j + b_k$, $3 \le j<k\le n$, 
are also neighbors of $b_2+b_1$. Hence, $b_2+b_1$ has $1+2 \times (n-2)+\binom{n-2}{2}$ neighbors in $\CBZERO$.
Then, the points $b_1+b_2+b_{j_1} $, $3 \leq j_1 \leq n$, have $1+2 \times (n-2-1) + \binom{n-2-1}{2}$ neighbors of this kind, 
using the same arguments, and there are $\binom{n-2}{1}$ ways to chose $b_{j_1}$. 
In general, there are $\binom{n-2}{i}$ ways to form a $1+2 \times (n-2-i) + \binom{n-2-i}{2}$~-~simplex.

To summarize, each pattern replicates $\sum_i \binom{n-2}{i}$ times, where at each step $i$ the patterns yield respectively $(l)$ $1+(n-2-i)$-simplices and $(ll)$ $1+2 \times (n-2-i) + \binom{n-2-i}{2}$-simplices.
As a result, the total number of pieces of $f$ is obtained as
\begin{equation}
\label{eq_proof_big}
\sum_{i=0}^{n-2} \left(   \underset{(l)}{\underbrace{\left[ 1+(n-2-i) \right]}}+  \underset{(ll)}{\underbrace{\left[\underset{(1)}{\underbrace{1+2(n-2-i)}}+ \underset{(2)}{\underbrace{\binom{n-2-i}{2}}} \right]}} \right) 
\times \underset{(o)}{\underbrace{\ \binom{n-2}{i}}}-1,
\end{equation}
where the -1 comes from the fact that for $i=n-2$, the piece generated by $(l)$ and the piece generated by $(ll)$ are the same. Indeed, the bisector hyperplane of $x$, $x+b_1$ and the bisector hyperplane of $x+b_2$, $x+b_2+b_1$ are the same since $b_2$ and $b_1$ are perpendicular.
\end{proof}
\subsection{Proof of Theorem~\ref{theo_Dn_second_kind_lin}: 
folding of $f$ with the second basis of $D_n$}
\label{App_folding_second_kind}

\begin{lemma}
\label{lem_fold_2}
Among the elements of $\mathcal{C}_{\PPB}$, only the points
of the form 
\begin{enumerate}
\item $x_1=b_3+...+b_{i-1}+b_i$ and $x_1 +b_1$, 
\item $x_2=b_3+...+b_{i-1}+b_i+b_2$ and $x_2 + b_1$,
\end{enumerate}
$i \leq n$, 
are on the non-negative side of all $BH(b_j,b_k)$, $3\le  j < k \le~n$.  
\end{lemma}
\begin{proof}
See the proof of Lemma~\ref{lem_fold}.
\end{proof}
\begin{proof} (of Theorem~\ref{theo_Dn_second_kind_lin})
(i) The folding  via $BH(b_j, b_k)$, $3 \leq j < k  \leq n$, switches $b_j$  and $b_k$ in
the hyperplane containing $\D(\B)$, which is orthogonal to $e_1$. Switching $b_j$
and $b_k$ does  not change the decision boundary because of the basis
symmetry, hence  $f$ is unchanged. 

Now, for (ii), how many pieces are left after all reflections?
To count the number of pieces of $f$, defined on $\D'(\B)$, 
we need to enumerate the cases where both $x \in \CBONE$ and $x'\in \CN(x) \cap \CBZERO$ are 
on the non-negative side of all reflection hyperplanes. 

Firstly, we investigate the effect of the folding operation on the term $\sum_{i=0}^{n-2}[1+(n-2-i)] \times \binom{n-2}{i}$ in Equation~\eqref{eq_proof_big}.
Remember that it is obtained via $(l)$ (i.e.  Equation~\eqref{eq_prop_Dn_2}).
Due to the reflections, among the points in $\CBONE$ of the form $\sum_j b_j+b_1$ only $x=b_3+b_4+...+b_{i-1}+b_i+b_1$, $j\leq n$, is on the non-negative side
of all reflection hyperplanes (see result 1. of Lemma~\ref{lem_fold_2}).
Similarly, among the elements in $\CN(x) \cap \CBZERO$, only $x-b_1$ and $x-b_1+b_{i+1}$ (instead of  $x-b_1+b_k$, $3\leq k \leq n$)  are on the non-negative side of all reflection hyperplanes.
Hence, at each step $i$, the term $[1+(n-2-i)]$ becomes 2 (except for $i=n-2$ where it is 1).
Therefore, the folding operation reduced the term $\sum_{i=0}^{n-2}[1+(n-2-i)] \times \binom{n-2}{i}$ 
to $(n-2)\times2 + 1$.

Secondly, we investigate the reduction of the term $\sum_{i=0}^{n-2}\left[1+2(n-2-i)+\binom{n-2-i}{2}\right] \times \binom{n-2}{i}$ obtained via $(ll)$ (i.e.  Equation~\ref{eq_prop_Dn_2__2}).
The following results are obtained via item 2. of Lemma~\ref{lem_fold_2}.
Among the points denoted by $\sum_j b_j+b_2+b_1 \in \CBONE$ only $x=b_3+b_4+...+b_{i-1}+b_i+b_2+b_1$ is on the proper side of all reflection hyperplanes.
Among the neighbors of any of these points, of the form $(ll)-(2)$, only $x+b_{i+1}+b_{i+2}$ is on the proper side of all hyperplanes.
Additionally, among the neighbors of the form $(ll)-(1)$ and $(ll)-(b)$, i.e. $x+b_k$ or $x-b_2+b_k$, $3\leq k \leq n$, $b_k$ can only be $b_{i+1}$.
Therefore, the folding operation reduces the term  $\sum_{i=0}^{n-2}[1+2(n-2-i)+\binom{n-2-i}{2}] \times \binom{n-2}{i}$ 
to $(n-3)\times4 + 3+1$.

\end{proof}

\subsection{Proof of Theorem~\ref{theo_nbReg_Lin_En}: number of pieces of $f$ for $E_n$}
\label{App_func_En}
\begin{proof}
Similarly to the proof of Theorem~\ref{eq_nbReg_Dn_second_kind}, 
we count the number of simplices and investigate their size.
The number of pieces of $f$ is then obtained by summing 
the number of pieces of the boundary functions of each simplex 
(again, see the proof of Theorem~\ref{theo_nbReg_Lin_Dn_const_A}).
This is achieved by counting the number of elements in $\CN(x+~b_1)~\cap~\CBZERO$, for all $x \in \CBZERO$. In this scope, we group the lattice points $x \in \CBZERO$ within three categories.
The numbering of these categories matches the one given 
in the sketch of proof (see also Equation~\ref{eq_En_bis} below). $\sum_j b_j$ denotes any sum of points in the set $\{0,b_j\}_{j=4}^{n}$.

\begin{gather}
\label{eq_En_1}
\begin{split}
(l): \ &\forall \ x=\sum_j b_j \in~\CBZERO, \ x' \in D_n \backslash \{b_j,0\},\ 4 \leq k\leq n: \\
&x+ b_k \in \CN(x+b_1), \ x+ x' \not\in \CN(x+b_1) \cap \CBZERO. \\
\end{split}
\end{gather}
\begin{gather}
\label{eq_En_2}
\begin{split}
(ll)-A \ &\forall \ x= \sum_j b_j + b_2\in~\CBZERO, \ x' \in D_n \backslash \{b_i,-b_2 + b_i,-b_2 + b_i +b_k,0\},\ 4 \leq i < k\leq n: \\
&(1)  \ x+ b_i \in \CN(x+b_1), \ x -b_2 + b_i \in \CN(x+b_1), \\
&(2) \ x -b_2 + b_i +b_k  \in \CN(x+b_1), \\
&(3) \ x+ x' \not\in \CN(x+b_1) \cap \CBZERO. 
\end{split}
\end{gather}
\begin{gather}
\label{eq_En_3}
\begin{split}
(ll)-B \ &\forall \ x= \sum_j b_j + b_3\in~\CBZERO, \ x' \in D_n \backslash \{b_i,-b_3 + b_i,-b_3 + b_i +b_k,0\},\ 4 \leq i < k\leq n: \\
&(1) \ x+ b_i \in \CN(x+b_1), \ x -b_3 + b_i \in \CN(x+b_1), \\
&(2) \ x -b_3 + b_i +b_k  \in \CN(x+b_1), \\
&(3) \ x+ x' \not\in \CN(x+b_1) \cap \CBZERO. 
\end{split}
\end{gather}
\begin{gather}
\label{eq_En_4}
\begin{split}
(lll) \ &\forall \ x= \sum_j b_j + b_2+b_3 \in~\CBZERO, x' \in D_n \backslash \{b_i,b_i + b_k,b_i + b_k + b_l,0\}, \ 4 \leq i < k <l\leq n: \\
&(1) \ x -b_2+ b_k \in \CN(x+b_1), \ x -b_3+ b_k \in \CN(x+b_1), \ x  + b_k \in \CN(x+b_1),\\
&(2) \ x -b_3 - b_2 + b_i + b_k \in \CN(x+b_1), \\
& x  - b_2 + b_i + b_k \in \CN(x+b_1), \ x -b_3 + b_i + b_k \in \CN(x+b_1),\\
&(3) \ x+ b_i + b_k +b_l \in \CN(x+b_1), \\
& (4) \ x+ x' \not\in \CN(x+b_1) \cap \CBZERO. 
\end{split}
\end{gather}
We count the number of $i$-simplices per category.

$(l)$ is like $A_n$ (see Appendix~\ref{App_theo4}). Starting from the origin, one can form a $n-2$-simplex with $0$, $b_1$, and the $n-2$ other basis vectors except $b_2$ and $b_3$ (because they are perpendicular to $b_1$). 
Then, from any $b_{j_1}$, $4 \leq j_1\leq n$, one can only add (to $b_{j_1}$) the $n-3$ remaining basis vectors to generate a simplex in $\PPB$ where the top corner is $b_{j_1}+b_1$. 
Indeed, if we add again~$b_{j_1}$, the resulting point (i.e. $b_{j_1} + b_{j_1}$) is outside $\PPB$. 
Hence, we get a $n-3$-simplex and there are $\binom{n-3}{1}$
ways to choose $b_{j_1}$: any basis vector except~$b_1,b_2,b_3$. 
Similarly, if one starts the simplex from $b_{j_1}+b_{j_2}$, one can form a $n-4$-simplex in $\PPB$ and there are $\binom{n-3}{2}$ ways to choose $b_{j_1}+b_{j_2}$. 
In general, there are $\binom{n-3}{i}$ ways to form a $n-2-i$-simplex.

$(ll)$ is like the second basis of $D_n$ (see $(ll)$ in the proof in Appendix~\ref{App_Dn_second_kind}), repeated twice because we now have two basis vectors orthogonal to $b_1$ instead of one.
Hence, we get that there are $\binom{n-3}{i}$ ways to form a $2 \times \left(1+2(n-3-i)+ \binom{n-3-i}{2}\right)$-simplex.

$(lll)$ is the new category. We investigate the neighbors of a given point $x=\sum_j b_j+ b_3+b_2+b_1$.
First (1), any $\sum_j b_j+b_3+ b_2$ is in $\CN(x)\cap \CBZERO$. Any $\sum_j b_j + b_2+b_k$, $\sum_j b_j+b_3+b_k$, and $\sum_j b_j+ b_3+b_2+b_k$, where  $4\leq k \leq n$ and $k \not\in \{ j\}$ are also in $\CN(x)\cap \CBZERO$. 
Hence, there are $3 \times (n-3-i)$ of such neighbors, where $i = |\{j\}|$ (in $\sum_j b_j$). Then, (2) any $\sum_j b_j + b_i+b_k$, $\sum_j b_j + b_2+ b_i+b_k$, and $\sum_j b_j + b_3+ b_i+b_k$,  where $4\leq i<k \leq n$ and $i,k \not\in \{ j\}$,
are in $\CN(x)\cap \CBZERO$. There are $3\times \binom{n-3-i}{2}$ possibilities, where $i = |\{j\}|$. 
Finally (3), any $\sum_j b_j + b_i+ b_k+b_l$, $4 \leq i<k<l \leq n$ and $i,k,l \not\in \{ j\}$ are in $\CN(x)\cap \CBZERO$. There are $\binom{n-3-i}{3}$ of them, where $i = |\{j\}|$.

To summarize, each pattern replicates $\sum_i \binom{n-3}{i}$ times, where at each step $i$ the patterns yield $(l)$ $1+n-3-i$-simplices, $(ll)$  $2 \times \left(1+2(n-3-i)+ \binom{n-3-i}{2}\right)$-simplices, and $(lll)$
$1 + 3 \times (n-3-i) + 3\times \binom{n-3-i}{2}+ \binom{n-3-i}{3}$-simplices. As a result, the total number of pieces of $f$ is obtained as
\tiny
\begin{equation}
\label{eq_En_bis}
\sum_{i=0}^{n-3} \left( \underset{(l)}{\underbrace{\left[ 1+(n-3-i) \right]}} +\underset{(ll)}{\underbrace{2 \left[ 1+ 2(n-3-i)+ \binom{n-3-i}{2} \right]}} + \underset{(lll)}{\underbrace{\left[ \underset{(1)}{\underbrace{1+3(n-3-i) }}+  \underset{(2)}{\underbrace{3\binom{n-3-i}{2}}} + \underset{(3)}{\underbrace{\binom{n-3-i}{3}}} \right]} }\right)\underset{(o)}{\underbrace{\binom{n-3}{n-i}}}-3,
\end{equation}
\normalsize
where the -3 comes from the fact that for $i=n-3$, the four pieces generated by $(l)$, $(ll)$, and $(lll)$ are the same. 
Indeed, the bisector hyperplane of $x$, $x+b_1$, is the same as the one of $x+b_2$, $x+b_2+b_1$, of $x+b_3$, $x+b_3+b_1$,
and of $x+b_2+b_3$, $x+b_2+b_3+b_1$, since both $b_2$ and $b_3$ are perpendicular to $b_1$.
\end{proof}

\subsection{Proof of Theorem~\ref{theo_En_folding}: folding of $f$ for $E_n$}
\label{App_folding_En}

\begin{lemma}
\label{lem_fold_3}
Among the elements of $\mathcal{C}_{\PPB}$, only the points
of the form 
\begin{enumerate}
\item $x_1=b_4+...+b_{i-1}+b_i$ and $x_1 +b_1$, 
\item $x_2=b_4+...+b_{i-1}+b_i+b_2$ and $x_2+b_1$,
\item $x_3=b_4+...+b_{i-1}+b_i+b_2+b_3$ and $x_3+b_1$, 
\end{enumerate}
$i \leq n$, 
are on the non-negative side of all $BH(b_j,b_k)$, $4\le  j < k \le~n$.  
\end{lemma}
\begin{proof}
See the proof of Lemma~\ref{lem_fold}.
\end{proof}

\begin{proof}(of Theorem~\ref{theo_En_folding})
(i) The folding via $BH(b_j, b_k)$, $4 \leq j < k  \leq n$ and $j=2,k=3$, switches $b_j$  and $b_k$ in
the hyperplane containing $\D(\B)$, which is orthogonal to $e_1$. Switching $b_j$
and $b_k$ does  not change the decision boundary because  of the basis
symmetry, hence $f$ is unchanged. 

Now, for (ii), how many pieces are left after all reflections?
To count the number of pieces of $f$, defined on $\D'(\B)$, 
we need to enumerate the cases where both $x \in \CBONE$ and $x'\in \CN(x) \cap \CBZERO$ are 
on the non-negative side of all reflection hyperplane. 

Firsly, we investigate the effect of the folding operation on the term $\sum_{i=0}^{n-3}[1+n-3-i] \times \binom{n-3}{i}$ in Equation~\eqref{eq_En_bis}.
Remember that it is obtained via $(l)$ (i.e.  Equation~\eqref{eq_En_1}).
Due to result 1 of Lemma~\ref{lem_fold_3} and similarly to the corresponding term in the proof of Theorem~\ref{theo_Dn_second_kind_lin}, this term reduces to
$(n-3)\times2 + 1$.

Secondly, we investigate the reduction of the term $2 \left[ 1+ 2(n-3-i)+ \binom{n-3-i}{2} \right] \times \binom{n-3}{i}$, obtained via $(ll)$ (i.e.  Equation~\ref{eq_En_2}).
The following results are obtained via item 2 of Lemma~\ref{lem_fold_3}. 
$\binom{n-3}{i}$ reduces to 1 at each step $i$ because in $\CBONE$, only the points $x=b_2+b_3+b_{i-1}+b_i+b_1$ are on the non-negative side of all hyperplanes, $i\leq n$.
Then, since any $\sum_j b_j + b_3 +b_1$ is on the negative side of the hyperplane $BH(b_2,b_3)$, $(ll)-(B)$ generates no pieces in $f$ (defined to $\D'(\B)$).
$(ll)-(A)$ is the same situation as the situation $(ll)$ in the proof of Theorem~\ref{theo_Dn_second_kind_lin}.
Hence, the term reduces to $(n-3)\times(4)+3+1$.

Finally, what happens to the term $\left[1+3(n-3-i) +  3\binom{n-3-i}{2} + \binom{n-3-i}{3} \right] \binom{n-3}{n-i}$, obtained via $(lll)$ (i.e.  Equation~\ref{eq_En_3})?
The following results are obtained via item 3 of Lemma~\ref{lem_fold_3}. 
As usual, $\binom{n-3}{n-i}$ reduces to 1 at each step $i$.
Then, $3(n-3-i)$, due to $(lll)-(1)$, becomes $2\times 1$ at each step $i$ because any $x-b_2+b_k$ (in $(lll)-(1)$), $k \leq 4 \leq n$, is on the negative side of $BH(b_2,b_3)$. 
For $x-b_3+b_k$ and $x+b_k$, only one valid choice of $b_k$ remains at each step $i$, as explained in the proof  of Theorem~\ref{theo_Dn_second_kind_lin}.
Regarding the term $3\binom{n-3-i}{2}$, due to $(lll)-(2)$, any point $x-b_2+b_i+b_k$ (in $(lll)-(2)$) is on the negative side of $BH(b_2,b_3)$ and 
at each step $i$ there is only one valid way to
chose $b_j$ and $b_k$ for both $x-b_3-b_2+b_j+b_k$ and $x-b_3+b_j+b_k$.
Eventually, for the last term due to $(lll)-(3)$ only one valid choice remain at each step $i$.
Therefore, the term due to $(lll)$ is reduced to
to $(n-4)\times6 + 5+3+1$.
\end{proof}

\section{From folding to a deep ReLU network}
\label{sec_folding_relu}

For the sake of simplicity and without loss of generality, in addition to the standard ReLU activation function ReLU$(a)=\max(0, a)$, we also allow the function $\max(0,-a)$ and the identity as activation functions in the network.

To implement a reflection, one can use the following strategy.
\begin{enumerate} 
\item Step~1: rotate the axes to have the $i$-th axis $e_i$ perpendicular to the reflection hyperplane and shift the point (i.e. the $i$-th coordinate) to have the reflection hyperplane at the origin.
\item Step~2: compute the absolute value of the $i$-th coordinate. 
\item Step~3: do the inverse operation of step 1.
\end{enumerate}
Now consider the ReLU network illustrated in Figure~\ref{fig_relu_ref}. 
The edges between the input layer and the hidden layer represent the rotation matrix, where the $i$-th column is repeated twice, and $p$ is a bias applied on the $i$-th coordinate.
Within the dashed square, the absolute value of the $i$-th coordinate is computed
and shifted by $-p$.
Finally, the edges between the hidden layer and the output layer represent the inverse rotation matrix. 
This ReLU network computes a reflection. We call it a reflection block.

\begin{figure}[H]
    \centering
    \includegraphics[scale=0.75]{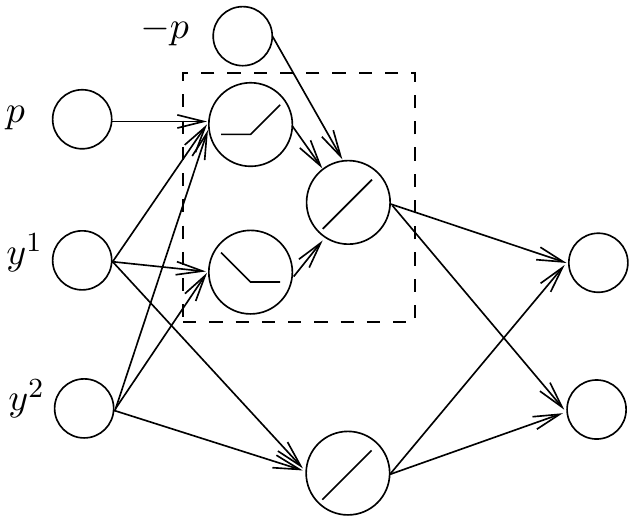}
     \caption{Reflection ReLU network (called reflection block).}
     \label{fig_relu_ref}
\end{figure}

All reflections can be naively implemented by a simple
concatenation of reflection blocks.  
Since the $\mathcal{O}(n)$ remaining operations to perform are negligible compared to the previous folding operations (see e.g. Appendix~\ref{App_remaining_pieces}),
the depth of this network increases linearly with the number of reflections and its width is linear in the dimension.


\section{Additional material}
\label{App_addi_mat}
The following figure shows the decision boundary function $f^{n=3}$ defined on $\D(B)$ oscillating around the hyperplane $\Phi^{n=3}$.
We can also observe the truncated simplices.
The black edges on the figure connect a point $x \in \CBONE$ to an element of $\CN(x) \cap \CBZERO$. 
Any piece of $f$ is orthogonal to one of these edges.
\begin{figure}[H]
    \centering
    \includegraphics[scale=0.5]{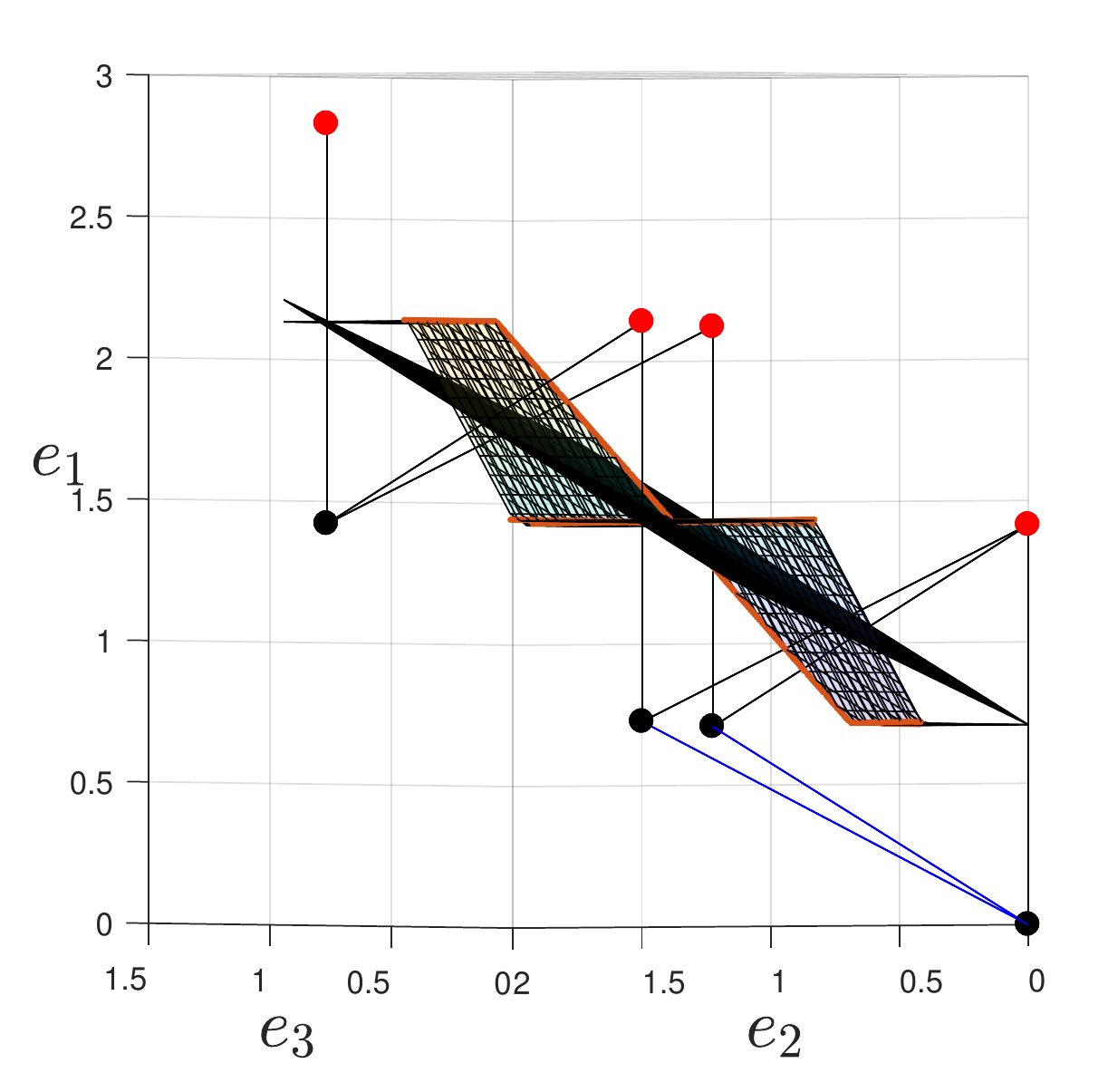}
     \caption{Decision boundary function for $A_3$ oscillating around the plane $\Phi^{n=3}$.}
     \label{fig_func_A3_other}
\end{figure}

\begin{figure}[H]
    \centering
    \includegraphics[scale=0.5]{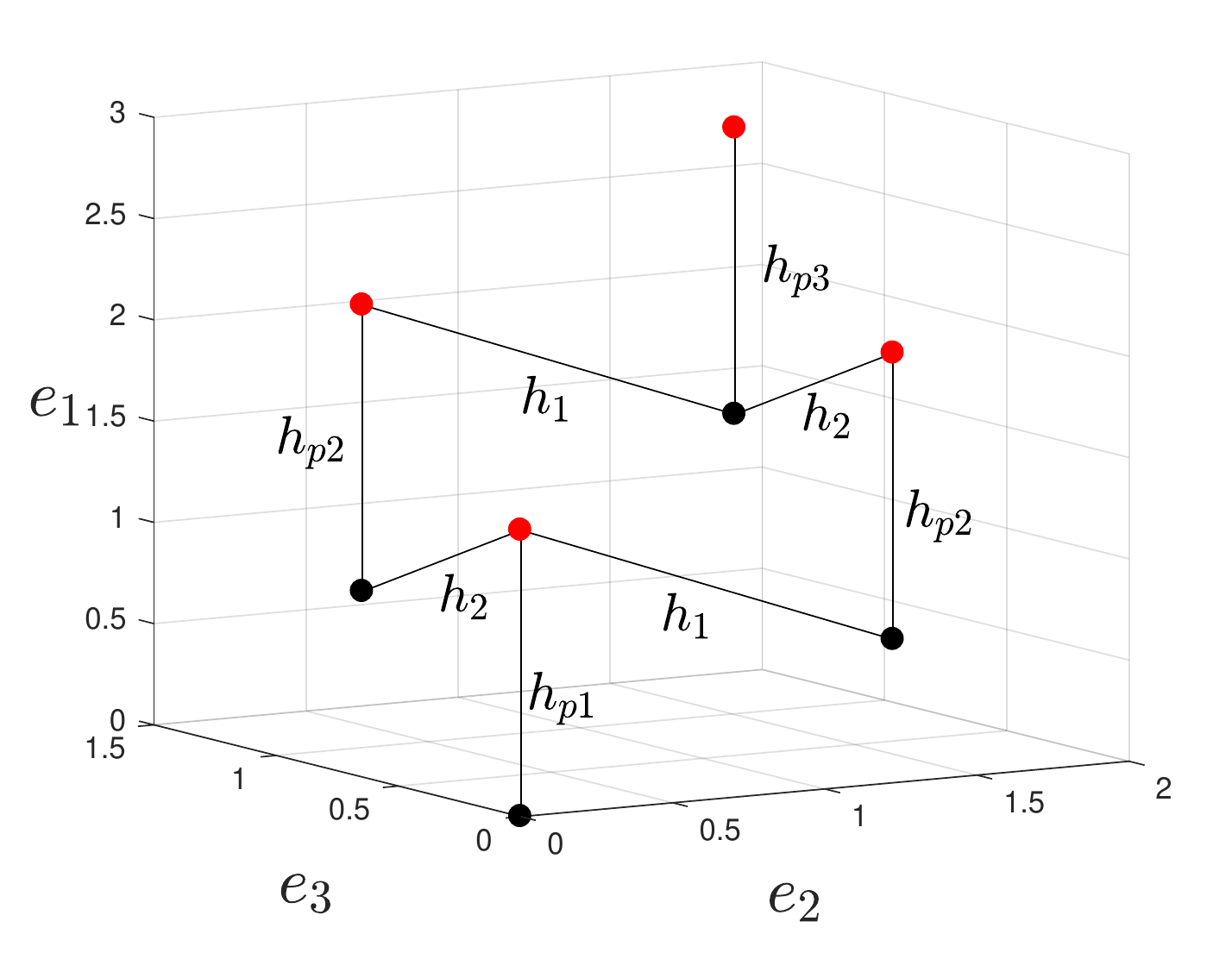}
     \caption{``Neighbor" figure of $\mathcal{C}_{\PPB}$ for $A_3$. Each edge connects a point $x \in \CBONE$ to an element of $\CN(x) \cap \CBZERO$.
	          The $i$ edges connected to a point $x \in \CBONE$ are $1$-faces of a regular $i$-simplex. The pieces of $f$ are orthogonal to these edges.}
     \label{fig_simplex_A3_neighbor}
\end{figure}

\section{Theorem~4 of \cite{Corlay2019} and its proof}
\label{App_theo4}

The purpose of this part is to count the number of pieces, and thus linear regions,
of the decision boundary function $f$ for $A_n$.
We start with the following lemma involving $i$-simplices, which is then used to prove the following theorem.

\begin{lemma}
\label{lem_simplex}
Consider an $A_n$-lattice basis defined by the Gram matrix~\eqref{eq_basis_An}.
The decision boundary function $f$ has a number of affine pieces equal to
\begin{align}
\sum_{i=0}^{n}i \times ( \text{$\#$ regular $i$-simplices}),
\end{align}
where, for each $i$-simplex, only one corner $x$ belongs to $\CBONE$
and the other corners constitute the set $\CN(x) \cap \CBZERO$.
\end{lemma}
\begin{proof}
A key property of this basis is 
\begin{gather}
\label{eq_prop}
\begin{split}
&\forall \ x \in~\CBZERO, \ x' \in A_n \backslash \{b_j,0\},\ 2 \leq j\leq n: \\
&x+ b_j \in \CN(x+b_1), \ x+ x' \not\in \CN(x+b_1) \cap \CBZERO.
\end{split}
\end{gather}
It is obvious that $\forall \ x \in~\CBZERO$: $x + b_1 \in~\CBONE$. 
This implies that any given point $x \in \mathcal{C}^1_{\PPB}$
and its neighbors $\CN(x)~\cap~\CBZERO$ form a regular simplex
$\mathcal{S}$ of dimension $|\CN(x) \cap \CBZERO|$. 
This clearly appears on Figure~\ref{fig_simplex_A3_neighbor}.
Now, consider the decision boundary function of a $k$-simplex separating the top corner (i.e. $\mathcal{C}^1_{\mathcal{S}}$) from all the other corners  (i.e. $\mathcal{C}^0_{\mathcal{S}}$). 
This function is convex and has $k$ pieces. 
The maximal dimension of such simplex is obtained
by taking the points 0, $b_1$, and the $n-1$ points $b_j$, $j\ge 2$.
\end{proof}

\begin{theorem}
\label{theo_nbReg_Lin}
Consider an $A_n$-lattice basis defined by the Gram matrix~\eqref{eq_basis_An}.
The decision boundary function $f$ has a number of affine pieces equal to
\begin{equation}
\label{eq_An_pieces}
\sum_{i=1}^{n}i \cdot \binom{n-1}{n-i}.
\end{equation}
\end{theorem}

\begin{proof}
From Lemma~\ref{lem_simplex},
what remains to be done is to count the number of $k$-simplices.
We walk in $\CBZERO$ and for each of the $2^{n-1}$ points $x \in \CBZERO$ 
we investigate the dimension of the simplex where the top corner is $x + b_1 \in \CBONE$. 
This is achieved by counting the number of elements in $\CN(x+~b_1)~\cap~\CBZERO$, 
via the property given by (\ref{eq_prop}). 
Starting from the origin, one can form a $n$-simplex with $0$, $b_1$, and the $n-1$ other basis vectors. 
Then, from any $b_{j_1}$, $j_1\neq 1$, one can only add the $n-1$ remaining basis vectors to generate a simplex in $\PPB$. 
Indeed, if we add again~$b_{j_1}$, the point goes outside $\PPB$. 
Hence, we get a $n-1$-simplex and there are $\binom{n-1}{1}$
ways to choose $b_{j_1}$: any basis vector except~$b_1$. 
Similarly, if one starts the simplex from $b_{j_1}+b_{j_2}$, one can form a $n-2$-simplex in $\PPB$ and there are $\binom{n-1}{2}$ ways to choose $b_{j_1}+b_{j_2}$. 
In general, there are $\binom{n-1}{k}$ ways to form a $n-k$-simplex.
Applying the previous lemma and summing over $k=n-i=0 \ldots n-1$ gives
the announced result.
\end{proof}

\section{Computing the $\mathcal{O}(n)$ remaining pieces of $f$ after folding}
\label{App_remaining_pieces}

The remaining pieces of $f$ can be evaluated via $\mathcal{O}(\log(n))$ additional hidden layers. 
First, compute the $\mathcal{O}(n)$ $\vee$ via $\mathcal{O}(1)$ layers of size $\mathcal{O}(n)$ containing several ``max ReLU networks" 
(see e.g. Figure~3 in \cite{Arora2018}). 
Then, compute the $n$-$\wedge$ via $\mathcal{O}(\log(n))$ layers. 

\section{Additional material on lattices}
\label{App_latt}

The Gram matrix is $\Gamma=GG^T=(GQ)(GQ)^T$, 
where $Q$ is any $n \times n$ orthogonal matrix.
All bases defined by a Gram matrix are equivalent modulo rotations and reflections.
A lower triangular generator matrix is obtained from the Gram matrix 
by Cholesky decomposition. 

$\PPB$ and $\mathcal{V}(x)$ are fundamental regions of the lattice: one can perform a tessellation of $\R^n$ with these regions.

The fundamental parallelotope of $\Lambda$, defined by a basis $\B$, is given by
\begin{equation}
\label{equ_PB}
\mathcal{P}(\mathcal{B}) = \{ y \in \R^n : y=\sum_{i=1}^{n}\alpha_{i}g_{i}, \ 0 \leq \alpha_{i} < 1  \}.
\end{equation}
The fundamental volume of $\Lambda$ is $\det(\Lambda)=|\det(G)|=\vol(\mathcal{V}(x))=\vol(\mathcal{P}(\mathcal{B}))$.
The Voronoi cell of $x$ is:
\begin{equation}
\mathcal{V}(x)=\{ y \in \R^n : \|y-x\| \le \|y-x'\|, \forall \ x' \in \Lambda \}.
\end{equation}

A surface in $\R^n$ defined by a function $g$ of $n-1$ arguments is written as
$\text{Surf}(g)=\{ (g(\tilde{y}), \tilde{y}) \in \R^n : \tilde{y} \in \R^{n-1} \}$.

\begin{definition}
\label{def_semi-Voronoi-reduced}
Let $\B= \{b_i\}_{i=1}^{n}$ be a basis of $\Lambda$. Suppose that the $n-1$ points $\mathcal{B} \backslash \{ b_1\}$
belong to the hyperplane $\{y \in \R^n : \ y \cdot e_1 =0\}$. The basis is called
semi-Voronoi-reduced (SVR) if there exists at least two points $x_1, x_2 \in \CBONE$
such that 
$\text{Surf}(\vee_{k=1}^{\ell_1}g_{1,k}) \bigcap \text{Surf}(\vee_{k=1}^{\ell_2}g_{2,k})
\ne \varnothing$, where $\ell_1,\ell_2\geq 1$,
$g_{1,k}$ are the facets between $x_1$ and all points in $\CN(x_1) \cap \CBZERO$,
and $g_{2,k}$ are the facets between $x_2$ and all points in $\CN(x_2) \cap \CBZERO$.
\end{definition}

\end{document}